\documentclass{article}

\usepackage{arxiv}

\usepackage[utf8]{inputenc} 
\usepackage[T1]{fontenc}    
\usepackage[hidelinks]{hyperref}       
\usepackage{url}            
\usepackage{microtype}      
\usepackage{doi}
\usepackage{multirow}

\usepackage{graphicx}
\usepackage[acronym]{glossaries}
\usepackage{colortbl}
\usepackage{nicefrac}
\usepackage{subcaption}
\usepackage{booktabs}
\usepackage{xcolor}
\usepackage[linesnumbered,ruled,vlined]{algorithm2e}
\usepackage{lipsum}
\usepackage{wrapfig}
\usepackage{fontawesome}
\usepackage{tikzit}
\DontPrintSemicolon

\definecolor{myorange}{rgb}{0.906,0.435,0.317}
\definecolor{myblue}{rgb}{0.0,0.314,0.408}

\SetKwComment{Comment}{\color{green!50!black}// }{}

\SetKwProg{Function}{function}{}{}

\usepackage{amsmath, amsfonts, dsfont, amsthm}

\renewcommand{\min}[1]{\underset{#1}{\text{min}}\,}

\renewcommand{\inf}[1]{\underset{#1}{\text{inf}}\,}

\newcommand{\iid}[0]{\overset{iid}{\sim}}
\newcommand{\argmin}[1]{\underset{#1}{\text{argmin}}\,}
\newcommand{\argmax}[1]{\underset{#1}{\text{argmax}}\,}
\newcommand{\arginf}[1]{\underset{#1}{\text{arginf}}\,}

\newcommand{\trace}[0]{\text{Tr}}
\newcommand{\diag}[0]{\text{diag}}

\title{Lighter, Better, Faster Multi-Source Domain Adaptation with Gaussian Mixture Models and Optimal Transport}

\date{}

\newif\ifuniqueAffiliation
\uniqueAffiliationtrue

\ifuniqueAffiliation 
\author{
Eduardo Fernandes Montesuma\\
CEA, List\\
Université Paris-Saclay\\
F-91120 Palaiseau, France
\And
Fred Ngolè Mboula\\
CEA, List\\
Université Paris-Saclay\\
F-91120 Palaiseau, France
\And
Antoine Souloumiac\\
CEA, List\\
Université Paris-Saclay\\
F-91120 Palaiseau, France
}
\else
\usepackage{authblk}

\newbox{\orcid}\sbox{\orcid}{\includegraphics[scale=0.06]{orcid.pdf}} 
\author[1]{%
	\href{https://orcid.org/0000-0000-0000-0000}{\usebox{\orcid}\hspace{1mm}David S.~Hippocampus\thanks{\texttt{hippo@cs.cranberry-lemon.edu}}}%
}
\author[1,2]{%
	\href{https://orcid.org/0000-0000-0000-0000}{\usebox{\orcid}\hspace{1mm}Elias D.~Striatum\thanks{\texttt{stariate@ee.mount-sheikh.edu}}}%
}
\affil[1]{Department of Computer Science, Cranberry-Lemon University, Pittsburgh, PA 15213}
\affil[2]{Department of Electrical Engineering, Mount-Sheikh University, Santa Narimana, Levand}
\fi

\renewcommand{\headeright}{Gaussian Mixture Model Dictionary Learning}
\renewcommand{\undertitle}{}
\renewcommand{\shorttitle}{}

\hypersetup{
pdftitle={Gaussian Mixture Model Dictionary Learning},
pdfsubject={stat.ML},
pdfauthor={Eduardo Fernandes Montesuma,Fred Ngol\'e Mboula,Antoine Souloumiac},
pdfkeywords={Domain Adaptation,Optimal Transport Gaussian Mixture Models},
}

\newacronym{ot}{OT}{Optimal Transport}
\newacronym{eot}{EOT}{Empirical Optimal Transport}
\newacronym{dil}{DiL}{Dictionary Learning}
\newacronym{ml}{ML}{Machine Learning}
\newacronym{erm}{ERM}{Empirical Risk Minimization}
\newacronym{da}{DA}{Domain Adaptation}
\newacronym{tl}{TL}{Transfer Learning}
\newacronym{msda}{MSDA}{Multi-Source DA}
\newacronym{nmf}{NMF}{Nonlinear Matrix Factorization}
\newacronym{sota}{SOTA}{State-of-the-Art}
\newacronym{wdl}{WDL}{Wasserstein Dictionary Learning}
\newacronym{dadil}{DaDiL}{Dataset Dictionary Learning}
\newacronym{em}{EM}{Expectation-Maximization}
\newacronym{gmm}{GMM}{Gaussian Mixture Model}
\newacronym{mw}{MW}{Mixture Wasserstein}

\newacronym{tca}{TCA}{Transfer Component Analysis}
\newacronym{otda}{OTDA}{Optimal Transport Domain Adaptation}
\newacronym{sa}{SA}{Subspace Alignment}
\newacronym{coral}{CORAL}{Correlation Alignment}
\newacronym{jdot}{JDOT}{Joint Distribution Optimal Transport}
\newacronym{wbt}{WBT}{Wasserstein Barycenter Transport}

\newacronym{dann}{DANN}{Domain Adversarial Neural Network}
\newacronym{wdgrl}{WDGRL}{Wasserstein Distance Guided Representation Learning}
\newacronym{mcd}{MCD}{Maximum Classifier Discrepancy}
\newacronym{mdd}{MDD}{Margin Disparity Discrepancy}
\newacronym{wjdot}{WJDOT}{Weighted JDOT}
\newacronym{m3sda}{M3SDA}{Moment Matching for MSDA}

\newacronym{tsne}{t-SNE}{t-Stochastic Neighbor Embeddings}
\newacronym{nn}{NN}{Neural Net}
\newacronym{pot}{POT}{Python Optimal Transport}

\newacronym{map}{MAP}{Maximum a Posteriori}
\newacronym{sgd}{SGD}{Stochastic Gradient Descent}

\newacronym{jcpot}{JCPOT}{Joint Class Proportion and Optimal Transport}

\tikzstyle{trapezium}=[fill=white, draw=black, shape=trapezium, rotate=-90, minimum height=1cm]
\tikzstyle{lossbox}=[fill={rgb,255: red,202; green,206; blue,255}, draw=black, shape=rectangle, minimum height=1.2cm, minimum width=1cm, align=center]
\tikzstyle{clfbox}=[fill=white, draw=black, shape=rectangle, minimum width=1cm, minimum height=1cm]
\tikzstyle{new style 2}=[fill=white, draw=black, shape=rectangle, align=center]
\tikzstyle{domainbox}=[fill=white, draw=black, shape=rectangle, minimum width=3cm, align=center]
\tikzstyle{longbox}=[fill=white, draw=black, shape=rectangle, minimum height=4cm, minimum width=1.2cm, align=center]
\tikzstyle{rotatednode}=[rotate=90]
\tikzstyle{circularnode}=[fill=none, draw=black, shape=circle]
\tikzstyle{blue_circle}=[fill={rgb,255: red,0; green,80; blue,104}, draw=none, shape=circle, minimum width=0.5cm]
\tikzstyle{orangecircle1}=[fill={rgb,255: red,231; green,111; blue,81}, draw=none, shape=circle, minimum width=0.5cm]
\tikzstyle{blue_square1}=[fill={rgb,255: red,0; green,80; blue,104}, draw=none, shape=rectangle, minimum width=0.5cm, minimum height=0.5cm]
\tikzstyle{blue_triangle1}=[fill={rgb,255: red,0; green,80; blue,104}, draw=none, shape=regular polygon, regular polygon sides=3]
\tikzstyle{widebox}=[fill=white, draw=black, shape=rectangle, minimum height=1.2cm, minimum width=10cm, align=center]
\tikzstyle{labeled domain}=[fill=none, draw={rgb,255: red,0; green,80; blue,104}, shape=circle, minimum width=1cm]
\tikzstyle{unlabeled domain}=[fill=none, draw={rgb,255: red,231; green,111; blue,81}, shape=circle, minimum width=1cm]
\tikzstyle{smallwidebox}=[fill=white, draw=black, shape=rectangle, minimum height=1.2cm, minimum width=5cm, align=center]

\tikzstyle{red edge}=[->, fill=none, draw={rgb,255: red,128; green,0; blue,0}]
\tikzstyle{blue edge}=[->, fill=none, draw={rgb,255: red,70; green,130; blue,180}]
\tikzstyle{green edge}=[->, fill=none, draw={rgb,255: red,44; green,160; blue,44}]
\tikzstyle{red dotted edge}=[->, dashed, fill=none, draw={rgb,255: red,128; green,0; blue,0}, thick]
\tikzstyle{blue dotted edge}=[->, dashed, fill=none, draw={rgb,255: red,70; green,130; blue,180}]
\tikzstyle{green dotted edge}=[->, dashed, fill=none, draw={rgb,255: red,44; green,160; blue,44}, thick]
\tikzstyle{red dotted line}=[-, fill=none, dashed, draw={rgb,255: red,128; green,0; blue,0}]
\tikzstyle{blue dotted line}=[-, fill=none, dashed, draw={rgb,255: red,70; green,130; blue,180}]
\tikzstyle{green dotted line}=[-, fill=none, dashed, draw={rgb,255: red,44; green,160; blue,44}]
\tikzstyle{black edge}=[->]
\tikzstyle{black dashed line}=[-, dashed]
\tikzstyle{thick black arrow}=[->, thick]
\tikzstyle{thick black edge}=[-, thick]
\tikzstyle{thick black dotted line}=[->, thick, dashed]
\tikzstyle{semi transparent dashed black line}=[-, opacity=0.2, dashed]

\newtheorem{definition}{Definition}

\newtheorem{theorem}{Theorem}

\begin{document}
\maketitle

\begin{abstract}
In this paper, we tackle Multi-Source Domain Adaptation (MSDA), a task in transfer learning where one adapts multiple heterogeneous, labeled source probability measures towards a different, unlabeled target measure. We propose a novel framework for MSDA, based on Optimal Transport (OT) and Gaussian Mixture Models (GMMs). Our framework has two key advantages. First, OT between GMMs can be solved efficiently via linear programming. Second, it provides a convenient model for supervised learning, especially classification, as components in the GMM can be associated with existing classes. Based on the GMM-OT problem, we propose a novel technique for calculating barycenters of GMMs. Based on this novel algorithm, we propose two new strategies for MSDA: GMM-Wasserstein Barycenter Transport (WBT) and GMM-Dataset Dictionary Learning (DaDiL). We empirically evaluate our proposed methods on four benchmarks in image classification and fault diagnosis, showing that we improve over the prior art while being faster and involving fewer parameters\footnote{\faGithub~Our code is publicly available at \url{https://github.com/eddardd/gmm\_msda}}.
\end{abstract}
\section{Introduction}

Supervised learning models, especially deep neural nets, rely on large amounts of labeled data to learn a function that reliably predicts on unseen data. This property is known as \emph{generalization}. However, these models are subject to performance degradation, when the conditions upon which test data is acquired changes. This issue is known in the literature as distributional, or dataset shift~\cite{quinonero2008dataset}.

Under distributional shift, a possible solution is to acquire a new labeled dataset under the new conditions. This solution is, in many cases such as fault diagnosis~\cite{montesuma2022cross}, costly or infeasible. A different approach, known as \gls{da}, consists of collecting an unlabeled \emph{target domain} dataset, for which the knowledge in the \emph{source domain} dataset is transferred to~\cite{pan2009survey}. A way to further enhance this adaptation is to consider multiple related, but heterogeneous sources, which is known as \gls{msda}~\cite{crammer2008learning}.

In the context of \gls{da}, a prominent framework is \gls{ot}~\cite{villani2009optimal}, which is a field of mathematics concerned with the displacement of mass at least effort. This theory has been applied for \gls{da} in multiple ways, especially by (i) mapping samples between domains~\cite{courty2016optimal} and (ii) learning invariant representations~\cite{shen2018wasserstein}. For \gls{msda}, \gls{ot} has been used for aggregating the multiple source domains into a barycentric domain~\cite{montesuma2021icassp,montesuma2021cvpr}, which is later transported to the target domain, or by weighting source domain measures~\cite{turrisi2022multi}. Our work considers the problem of \gls{wdl}, initially proposed by~\cite{schmitz2018wasserstein} for histogram data. This problem was later generalized by~\cite{montesuma2023learning}, for empirical measures, which allowed its application to \gls{msda}. In~\cite{montesuma2023learning}, one expresses domains in \gls{msda} as a barycenter of atom measures, which have a free, learnable support. As a result, the work of~\cite{montesuma2023learning} \emph{learns how to interpolate distributional shift} between the measures in \gls{msda}.

\begin{figure}[ht]
    \centering
    \begin{subfigure}[t]{0.4\linewidth}
        \resizebox{0.9\linewidth}{!}{\input{gmm_wbt.tikz}}
        \caption{GMM-WBT}
    \end{subfigure}\hfill
    \begin{subfigure}[t]{0.52\linewidth}
        \resizebox{\linewidth}{!}{\input{gmm_dadil.tikz}}
        \caption{GMM-DaDiL}
    \end{subfigure}
    \caption{\textbf{Overview of proposed methods.} {\faDatabase} represent datasets, circles represent barycenters and triangles represent learned measures. Blue and orange elements represent labeled and unlabeled measures respectively. In \gls{gmm}-\gls{wbt}, a labeled \gls{gmm} is determined for the target domain by transporting the barycenter of sources. In \gls{gmm}-\gls{dadil}, we learn to express each domain as a barycenter of learned \glspl{gmm}, called atoms, through dictionary learning.}
    \label{fig:overview_methods}
\end{figure}

However, previous algorithms relying on Wasserstein barycenters, such as \gls{wbt}~\cite{montesuma2021icassp,montesuma2021cvpr} and \gls{dadil}~\cite{montesuma2023learning}, are limited in scale, since the number of points the support of the empirical measures scale with the number of samples in the original datasets. As a consequence, previous works such as~\cite{montesuma2021icassp,montesuma2021cvpr} are limited to small scale datasets, or rely on mini-batch optimization~\cite{montesuma2023learning}, which introduces artifacts in the \gls{ot}. To tackle these limitations, in this paper we propose a novel, parametric framework for barycentric-based \gls{msda} algorithms. based on \gls{ot} between \glspl{gmm}~\cite{delon2020wasserstein}. We present an overview of our methods in figure~\ref{fig:overview_methods}.

Our contributions are threefold: 1. We propose a novel strategy for mapping the parameters of \glspl{gmm} using \gls{ot} (section~\ref{sec:first_order_mw2} and theorem~\ref{thm:first_order_mw2}); 2. We propose a novel algorithm for computing mixture-Wasserstein barycenters of \glspl{gmm} (algorithm~\ref{alg:gmmot_bary} in section~\ref{sec:mixture_wbary}); 3. We propose an efficient parametric extension of the \gls{wbt} and \gls{dadil} algorithms based on \glspl{gmm} (section~\ref{sec:msda_gmm}). We highlight that, while \glspl{gmm} were previously employed in single source \gls{da}~\cite{gardner2022domain,montesuma2024gmmotda}, to the best of our knowledge this is the first work to leverage \gls{gmm}-\gls{ot} for \gls{msda}.

The rest of this paper is divided as follows. Section~\ref{sec:preliminaries} covers the background behind our method. Section~\ref{sec:methodological_contributions} covers our methodological contributions. Section~\ref{sec:experiments} explores the empirical validation of our method with respect other \gls{ot}-based \gls{msda} algorithms, where we show that our methods significantly outperform prior art. Finally, section~\ref{sec:conclusion} concludes this paper.
\section{Preliminaries}\label{sec:preliminaries}

\subsection{Gaussian Mixtures}

We denote the set of probability measures over a set $\mathcal{X}$ as $\mathbb{P}(\mathcal{X})$. A Gaussian measure corresponds to $P_{\theta} \in \mathbb{P}(\mathcal{X})$ with density,
\begin{align*}
    f_{\theta}(\mathbf{x}) &= \dfrac{1}{\sqrt{(2\pi)^{d}\det(\mathbf{C}^{(P)})}}\exp\biggr(-\dfrac{1}{2}(\mathbf{x}-\mathbf{m}^{(P)})^{T}(\mathbf{C}^{(P)})^{-1}(\mathbf{x}-\mathbf{m}^{(P)})\biggr),
\end{align*}
where $\theta = (\mathbf{m}^{(P)}, \mathbf{C}^{(P)})$ are the mean vector $\mathbf{m}^{(P)} \in \mathbb{R}^{d}$ and the covariance matrix $\mathbf{C}^{(P)} \in \mathbb{S}^{d} = \{\mathbf{C} \in \mathbb{R}^{d \times d}: \mathbf{C} = \mathbf{C}^{T}\text{ and }\mathbf{x}\mathbf{C}\mathbf{x}^{T} > 0, \forall \mathbf{x} \in \mathbb{R}^{d}\setminus \{\mathbf{0}\} \}$. We generally denote $P_{\theta} = \mathcal{N}(\mathbf{m}^{(P)}, \mathbf{C}^{(P)})$. In addition, let $K \geq 1$ be an integer. A \gls{gmm} over $\mathbb{R}^{d}$ is a probability measure $P_{\theta} \in \mathbb{P}(\mathbb{R}^{d})$ such that,
\begin{align}
    P_{\theta} = \sum_{k=1}^{K}p_{k}P_{k}\text{, where }P_{k} = \mathcal{N}(\mathbf{m}_{k}^{(P)},\mathbf{C}_{k}^{(P)})\text{, and }\mathbf{p} \in \Delta_{K},\label{eq:gmm}
\end{align}
where $\Delta_{K} = \{\mathbf{p} \in \mathbb{R}^{K}_{+}:\sum_{k=1}^{K}p_{k}=1\}$. Following~\cite{delon2020wasserstein}, we denote the subset of $\mathbb{P}(\mathbb{R}^{d})$ of probability measures which can be written as Gaussian mixtures with less than $K$ components by $\text{GMM}_{d}(K)$, and $\text{GMM}_{d}(\infty) = \cup_{k \geq 0}\text{GMM}_{d}(K)$.

Given data points $\{\mathbf{x}_{i}^{(P)}\}_{i=1}^{n}$ i.i.d. from $P$, one can determine the parameters $\theta$ through maximum likelihood,
\begin{align}
    \theta^{\star} &= \argmax{\theta \in \Theta} \sum_{i=1}^{n}\log P_{\theta}(\mathbf{x}_{i}^{(P)}),\label{eq:ml_gmm}
\end{align}
where $\Theta = \{\{p_{k},\mathbf{m}_{k}^{(P)},\mathbf{C}_{k}^{(P)}\}_{k=1}^{K}:\mathbf{m}_{k}^{(P)} \in \mathbb{R}^{d}\text{ and }\mathbf{C}_{k}^{(P)} \in \mathbb{S}^{d}\}$. While equation~\ref{eq:ml_gmm} has no closed-form solution, one can solve this optimization problem through the celebrated \gls{em} algorithm~\cite{dempster1977maximum}.

\subsection{Domain Adaptation}

In this paper, we focus on the problem of classification. Given a feature space $\mathcal{X} = \mathbb{R}^{d}$ and a label space $\mathcal{Y} = \{1,\cdots,n_{cl}\}$, this problem corresponds to finding $h \in \mathcal{H} \subset \mathcal{Y}^{\mathcal{X}}$ that correctly classifies data $\{(\mathbf{x}_{i}^{(Q)}, y_{i}^{(Q)})\}_{i=1}^{n}$.

We use the \gls{erm} framework~\cite{vapnik1991principles}, as it is useful for domain adaptation theory. As follows, one assumes $\mathbf{x}_{i}^{(Q)} \iid Q$, for a measure $Q \in \mathbb{P}(\mathcal{X})$, and $h_{0}:\mathcal{X}\rightarrow\mathcal{Y}$ such that $y_{i}^{(Q)} = h_{0}(\mathbf{x}_{i}^{(Q)})$. $h_{0}$ is called \emph{ground-truth labeling function}. Given a loss function $\mathcal{L}:\mathcal{Y}\times\mathcal{Y}\rightarrow\mathbb{R}$, a classifier may be defined through risk minimization, i.e., $h^{\star} = \text{argmin}_{h\in\mathcal{H}}\mathcal{R}_{Q}(h)$, for $\mathcal{R}_{Q}(h) = \mathbb{E}_{Q}[\mathcal{L}(h(\mathbf{x}), h_{0}(\mathbf{x}))]$. This strategy is oftentimes impractical as $Q$ and $h_{0}$ are unknown. As a result, one resorts to the minimization of the empirical risk, i.e., $\hat{h} = \text{argmin}_{h \in \mathcal{H}}\hat{\mathcal{R}}_{Q}(h)$, where $\hat{\mathcal{R}}_{Q}(h) = \dfrac{1}{n}\sum_{i=1}^{n}\mathcal{L}(h(\mathbf{x}_{i}^{(Q)}), y_{i}^{(Q)})$.

From a theoretical standpoint, this framework is useful because $\mathcal{R}_{Q}$ is bounded by $\hat{\mathcal{R}}_{Q}$ and a complexity term depending on the number of samples $n$, and the Vapnik-Chervonenkis dimension of $\mathcal{H}$~\cite[Section 6]{vapnik1991principles}. As a result, $\hat{h}$ minimizing the empirical risk is guaranteed to generalize to unseen samples of $Q$. Nevertheless, the assumption that unseen examples come from a fixed measure $Q$ is seldom verified in practice~\cite{quinonero2008dataset}, since the conditions upon which data is acquired may change. In this case, models are required to adapt to new data, but at the same time re-training a model from the scratch is likely costly and data intensive. A solution consists of using transfer learning~\cite{pan2009survey}, in which one re-uses knowledge from a source domain or task to facilitate the learning on a target domain or task.

In transfer learning, a domain is a pair $(\mathcal{X}, Q(X))$ of a feature space and a (marginal) probability measure. Likewise, a task is a pair $(\mathcal{Y}, Q(Y|X))$ of a label space and a conditional probability measure. Domain adaptation is a case in which one has two domains $(\mathcal{X}, Q_{S}(X))$, $(\mathcal{X}, Q_{T}(X))$, a single task $(\mathcal{Y}, Q(Y|X))$, and $Q_{S}(X) \neq Q_{T}(X)$. Furthermore, multi-source domain adaptation supposes multiple source domain measures, i.e., $Q_{S_{1}},\cdots,Q_{S_{N}}$, with $Q_{S_{i}} \neq Q_{S_{j}}$, and $Q_{S_{i}} \neq Q_{T}$. To reflect the idea that acquiring new data is costly, we have an unsupervised scenario. In this case, we have $N$ labeled source datasets $\{(\mathbf{x}_{i}^{(Q_{S_{\ell}})}, y_{i}^{(Q_{S_{\ell}})})\}_{i=1}^{n_{\ell}}$, and an unlabeled target dataset $\{\mathbf{x}_{i}^{(Q_{T})}\}_{i=1}^{n_{T}}$. Our goal is to learn a classifier on $Q_{T}$ by leveraging the knowledge from the source domains.

\subsection{Optimal Transport}

Optimal transport is a field of mathematics concerned with the displacement of mass at least effort~\cite{peyre2019computational,montesuma2023recent}. Given probability measures $P, Q \in \mathbb{P}(\mathcal{X})$, the Monge formulation~\cite[Section 2.2.]{peyre2019computational} of \gls{ot} seeks for a mapping $T$,
\begin{equation}
    T^{\star} = \arginf{T_{\sharp}P = Q}\int_{\mathcal{X}}c(x, T(x))dP(x),\label{eq:MongeFormulation}
\end{equation}
where $T_{\sharp}$ is the push-forward mapping of $T$, i.e., $T_{\sharp}P(A) = P(T^{-1}(A))$, and $c:\mathbb{R}^{d}\times\mathbb{R}^{d}\rightarrow\mathbb{R}$ is a ground-cost, that is, a measure of transportation effort. Nonetheless, this problem poses technical difficulties, mainly due the constraint $T_{\sharp}P = Q$. A more tractable formulation was proposed by Kantorovich~\cite[Section 2.3.]{peyre2019computational}, and relies on \gls{ot} plans,
\begin{equation}
    \gamma^{\star} = \arginf{\gamma\in\Gamma(P, Q)}\int_{\mathcal{X}}\int_{\mathcal{X}}c(x, z)d\gamma(x,z),\label{eq:KantorovichFormulation}
\end{equation}
where $\Gamma(P, Q) = \{\gamma \in \mathbb{P}(\mathcal{X}\times\mathcal{X}):\int_{\mathcal{X}}\gamma(A,z)=P(A)\text{, and }\int_{\mathcal{X}}\gamma(x,B)=Q(B)\}$ is called the transportation polytope. There is a metric between probability measures, associated with \gls{ot}, called Wasserstein distance~\cite{villani2009optimal}. As such, let $c(x,z) = d(x,z)^{\alpha}$ for $\alpha \in [1,\infty)$, where $d$ is a metric on $\mathcal{X}$, then,
\begin{equation}
    \mathcal{W}_{c,\alpha}(P, Q) = \biggr(\inf{\gamma\in\Gamma(P, Q)}\int_{\mathcal{X}}\int_{\mathcal{X}}c(x, z)d\gamma(x,z)\biggr)^{\nicefrac{1}{\alpha}}.
\end{equation}
When $\mathcal{X} = \mathbb{R}^{d}$, a common choice is $c(\mathbf{x},\mathbf{z}) = \lVert \mathbf{x} - \mathbf{z}\rVert_{2}^{\alpha}$, for which we omit the subscript $c$. Furthermore, common values for $\alpha$ include $1$ and $2$. Throughout this paper we adopt the Euclidean metric and $\alpha = 2$.

While equation~\ref{eq:KantorovichFormulation} is hard to solve for general $P$ and $Q$, it has closed-form solution for Gaussian measures~\cite{takatsu2011wasserstein}. As such, let $P = \mathcal{N}(\mathbf{m}^{(P)}, \mathbf{C}^{(P)})$ (resp. $Q$). Under these conditions, for $\mathbf{C}^{(P)} = \mathbf{S}^{(P)}(\mathbf{S}^{(P)})^{T}$,
\begin{equation*}
    \mathcal{W}_{2}(P, Q)^{2} = \lVert \mathbf{m}^{(P)} - \mathbf{m}^{(Q)} \rVert^{2}_{2} + \trace\biggr( \mathbf{C}^{(P)} + \mathbf{C}^{(Q)} - 2(\mathbf{S}^{(P)}\mathbf{C}^{(Q)}\mathbf{S}^{(P)})^{\nicefrac{1}{2}} \biggr),
\end{equation*}
This expression can be further simplified for axis-aligned Gaussians, i.e., $\mathbf{S}^{(P)} = \diag(\mathbf{s}^{(P)})$, with $\mathbf{s}^{(P)} \in \mathbb{R}^{d}_{+}$,
\begin{align}
    \mathcal{W}_{2}(P, Q)^{2} &= \lVert \mathbf{m}^{(P)} - \mathbf{m}^{(Q)} \rVert^{2}_{2} + \lVert 
\mathbf{s}^{(P)} - \mathbf{s}^{(Q)} \rVert_{2}^{2}.\label{eq:gauss_w2}
\end{align}
Henceforth, we assume axis-aligned Gaussian measures.

\noindent\emph{Remark.} Here, we give further insight into the hypothesis of using axis-aligned Gaussian measures. We use this assumption for numerical stability purposes, i.e., estimating the covariance matrix of \glspl{gmm} in high dimensions is much more difficult than estimating the standard deviation vector $\mathbf{s}^{(P)}$. Here, one has two choices. First, it is possible to introduce a transformation so as to force features to be uncorrelated (e.g., through principal components analysis). This approach, nonetheless, requires more data points per domain than features, which is not always feasible. Conversely, one can increase the number of components for expressing the shape of the data (see Fig.~\ref{fig:axis_aligned_gmm}). As we show in our experiments sections, we achieve good adaptation performance, even while sampling points from axis-aligned \glspl{gmm}.

\begin{figure}[ht]
    \centering
    \includegraphics[width=0.7\linewidth]{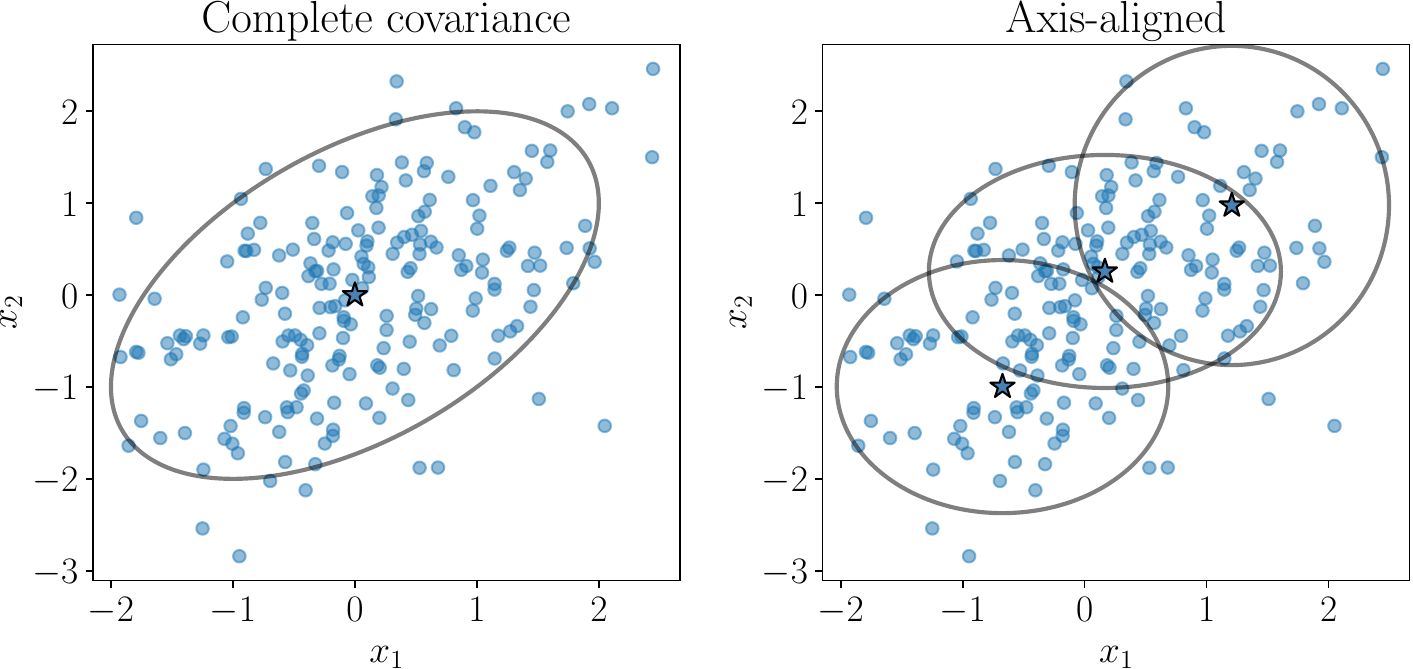}
    \caption{\textbf{Illustration of axis-aligned \glspl{gmm}.} This hypothesis leads to \glspl{gmm} that need more components to express the underlying data distribution.}
    \label{fig:axis_aligned_gmm}
\end{figure}

We use the \gls{gmm}-\gls{ot} framework of~\cite{delon2020wasserstein}, which is convenient to our setting for 2 reasons. First, they are able to represent measures with sub-populations, such as those commonly encountered in classification and domain adaptation. Second, they yield a tractable \gls{ot} problem, when $\gamma$ is further restricted to be a \gls{gmm} itself, that is,
\begin{align}
    \omega^{\star} &= \text{GMMOT}(P, Q) = \argmin{\omega \in \Gamma(\mathbf{p},\mathbf{q})}\sum_{i=1}^{K_{P}}\sum_{j=1}^{K_{Q}}\omega_{ij}\mathcal{W}_{2}(P_{i},Q_{j})^{2},\label{eq:gmmot}
\end{align}
where the \gls{ot} plan is given by $\gamma^{\star} = \sum_{i=1}^{n}\sum_{j=1}^{m}\omega_{ij}^{\star}f_{\theta}(\mathbf{x})\delta(\mathbf{y} - T_{ij}(\mathbf{x}))$ and $\mathcal{W}_{2}(P_{i}, Q_{j})^{2}$ is the Wasserstein distance between components $P_{i}$ and $Q_{j}$ (c.f., eq.~\ref{eq:gauss_w2}). Furthermore, the GMMOT problem defines the Mixture-Wasserstein distance~\cite{delon2020wasserstein},
\begin{align}
    \mathcal{MW}_{2}(P, Q)^{2} &= \sum_{i=1}^{K_{P}}\sum_{j=1}^{K_{Q}}\omega_{ij}^{\star}\mathcal{W}_{2}(P_{i},Q_{j})^{2}.\label{eq:mixture_wasserstein_distance}
\end{align}

\gls{dadil}~\cite{montesuma2023learning} is an \gls{ot}-based framework for expressing probability measures as barycenters of synthetic measures, called atoms. In this case, the authors use empirical measures, i.e., $\hat{P} = n^{-1}\sum_{i=1}^{n}\delta_{(\mathbf{x}_{i}^{(P)},\mathbf{y}_{i}^{(P)})}$. The framework is inspired by dictionary learning literature~\cite{schmitz2018wasserstein}. The authors introduce \emph{atoms} $\mathcal{P} = \{\hat{P}_{c}\}_{c=1}^{C}$ and \emph{barycentric coordinates} $\Lambda = \{\lambda_{\ell}\}_{\ell=1}^{N}$, such that each measure in \gls{msda} is expressed as a Wasserstein barycenter $\hat{Q}_{\ell} = \mathcal{B}(\lambda_{\ell},\mathcal{P})$. This framework leads to the following optimization problem,
\begin{align}
    (\Lambda^{\star}, \mathcal{P}^{\star}) &= \argmin{\Lambda,\mathcal{P}} \mathcal{W}_{2}(\hat{Q}_{T}, \mathcal{B}(\lambda_{\ell}; \mathcal{P})) + \sum_{\ell=1}^{N}\mathcal{W}_{c,2}(\hat{Q}_{\ell}, \mathcal{B}(\lambda_{\ell}; \mathcal{P}))^{2},
\end{align}
where $c\biggr{(}(\mathbf{x}^{(Q_{\ell})}, \mathbf{y}^{(Q_{\ell})}),(\mathbf{x}^{(B_{\ell})}, \mathbf{y}^{(B_{\ell})})\biggr{)} = \lVert \mathbf{x}^{(Q_{\ell})}
- \mathbf{x}^{(B_{\ell})} \rVert_{2}^{2} + \beta\lVert \mathbf{y}^{(Q_{\ell})} -
\mathbf{y}^{(B_{\ell})} \rVert_{2}^{2}$, and $\beta \geq 0$ is a constant expressing how costly it is to move samples from different classes. This framework makes it easier to express the distributional shift between the different $\mathcal{Q} = \{\hat{Q}_{1}, \cdots, \hat{Q}_{N_{S}}, \hat{Q}_{T}\}$. Especially, since $\hat{B}_{\ell} = \mathcal{B}(\lambda_{\ell},\mathcal{P})$ is labeled, one can synthesize labeled target domain data by reconstructing the target measure with $\lambda_{T} := \lambda_{N + 1}$.
\section{Methodological Contributions}\label{sec:methodological_contributions}

\subsection{First Order Analysis of $\mathcal{MW}_{2}$}\label{sec:first_order_mw2}

In this section we analyze $P \mapsto \mathcal{MW}_{2}(P,Q)^{2}$, for a fixed $Q$. We are particularly interested on how to map the components of $P$ towards $Q$, while minimizing this distance. The following theorem provides us a strategy,
\begin{theorem}\label{thm:first_order_mw2}
    Let $P$ and $Q$ be two \glspl{gmm} with components $P_{i} = \mathcal{N}(\mathbf{m}_{i}^{(P)}, (\mathbf{s}_{i}^{(P)})^{2})$ (resp. $Q_{j}$) and $\omega^{\star}$ be the solution of eq.~\ref{eq:gmmot}. The first-order optimality conditions of $\mathcal{MW}_{2}^{2}$, with respect $\mathbf{m}_{i}$ and $\mathbf{s}_{i}$ are given by,
    \begin{align}
        \hat{\mathbf{m}}_{i} = T_{\omega^{\star}}(\mathbf{m}_{i}^{(P)}) = \sum_{j=1}^{K_{Q}}\dfrac{\omega_{ij}^{\star}}{p_{i}}\mathbf{m}_{j}^{(Q)}\text{, and }
        \hat{\mathbf{s}}_{i} = T_{\omega^{\star}}(\mathbf{s}_{i}^{(P)}) = \sum_{j=1}^{K_{Q}}\dfrac{\omega_{ij}^{\star}}{p_{i}}\mathbf{s}_{j}^{(Q)},\label{eq:first_order_mw}
    \end{align}
    where $\omega^{\star}$ is the solution of eq.~\ref{eq:gmmot}.
\end{theorem}

\begin{proof}
    Our proof relies on the analysis of $\{(\mathbf{m}_{i},\mathbf{s}_{i})\}_{i=1}^{K_{P}} \mapsto \mathcal{MW}_{2}(P,Q)^{2}$ (c.f., equation~\ref{eq:mixture_wasserstein_distance}). Given $\omega^{\star} = \text{GMMOT}(P, Q)$,
    \begin{align*}
        \dfrac{\partial \mathcal{MW}_{2}^{2}}{\partial \mathbf{m}_{i}} = 2\sum_{j=1}^{K_{Q}}\omega_{ij}^{\star}(\mathbf{m}_{i}-\mathbf{m}_{j}^{(Q)}) = 2\biggr{(}p_{i}\mathbf{m}_{i} - \sum_{j=1}^{K_{Q}}\omega_{ij}^{\star}\mathbf{m}_{j}^{(Q)}\biggr{)},
    \end{align*}
    and, by equating this last term to $0$, one gets the desired equality.
\end{proof}

Equation~\ref{eq:first_order_mw} is similar to the barycentric mapping in \gls{eot}~\cite[eq. 13]{courty2016optimal}, which serves as an approximation for the Monge mapping between $P$ and $Q$. In our case, the barycentric mappings act on the parameters of the \gls{gmm}, rather than on its samples. Theorem~\ref{thm:first_order_mw2} will be useful in the calculation of $\mathcal{MW}_{2}$ barycenters.

\subsection{Supervised Mixture-Wasserstein Distances}\label{sec:supervised_gmm}

In this paper, we consider supervised learning problems. As such, it is necessary to equip the components of \glspl{gmm} with labels that represent the classes in the datasets. We propose doing so through a simple heuristic, especially, we model $P(\mathbf{x}|y)$ through a \gls{gmm}. We then concatenate the $n_{c}$ obtained \glspl{gmm}, and assign, for the $k-$th \gls{gmm} of the $y-$th class, $v_{k,y'}^{(P)} = \delta(y'-y)$, i.e., a vector of $n_{c}$ components, and $1$ on the $y-$th entry. We can assure that the resulting weights sum to $1$ by dividing their value by $\sum_{y=1}^{n_{c}}\sum_{k=1}^{K}p_{k,y}$, where $p_{k,y}$ corresponds to the weight of the $k-$th component of the $y-$th \gls{gmm}.

Given a \gls{gmm} $\{p_{k},\mathbf{m}_{k}^{(P)},\mathbf{s}_{k}^{(P)}, \mathbf{v}_{k}^{(P)}\}_{k=1}^{K}$, we define a classifier through \gls{map} estimation. This strategy is carried out through,
\begin{align}
    \hat{h}_{MAP}(\mathbf{x}) = \argmax{y=1,\cdots,n_{c}}P(y|\mathbf{x}) = \sum_{k=1}^{K}\underbrace{P_{\theta}(k|\mathbf{x})}_{p_{k}P_{k}(\mathbf{x})/\sum_{k'}p_{k'}P_{k'}(\mathbf{x})}\underbrace{P(y|k)}_{v_{k,y}^{(P)}},\label{eq:clf_map}
\end{align}
we use this classifier in a few illustrative examples in section~\ref{sec:exp_ablations}.

\noindent\emph{Remark.} In equation~\ref{eq:clf_map}, we are implicitly assuming that the component $k$ is conditionally independent with $y$ given $\mathbf{x}$. This remark is intuitive, as $\mathbf{x}$ explains, at the same time, the component and the label.

Similarly to \gls{eot}, when the mixtures $P$ and $Q$ are labeled, one needs to take into account the labels in the ground-cost. Given $\beta > 0$, we propose the following distance between labeled \glspl{gmm},
\begin{align}
    \mathcal{SMW}_{2}(P, Q)^{2} &= \min{\omega \in \Gamma(\mathbf{p},\mathbf{q})}\sum_{i=1}^{K_{P}}\sum_{j=1}^{K_{Q}}\omega_{ij}(\mathcal{W}_{2}(P_{i},Q_{j})^{2} + \beta\lVert \mathbf{v}_{i}^{(P)} - \mathbf{v}_{j}^{(Q)} \rVert_{2}^{2}).\label{eq:sup_mixture_wasserstein_distance}
\end{align}
While simple, using an Euclidean distance for the soft-labels allows us to derive similar first-order conditions for $\mathcal{SMW}_{2}$,
\begin{theorem}\label{thm:first_order_smw2}
    Under the same conditions of theorem~\ref{thm:first_order_mw2}, let $P_{i}$ and $Q_{j}$ be equipped with labels $\mathbf{v}_{i}^{(P)}$ and $\mathbf{v}_{j}^{(Q)}$. The first order optimality conditions of $\mathcal{SMW}_{2}$ with respect $\mathbf{m}_{i}$ and $\mathbf{s}_{i}$ are given by eq.~\ref{eq:first_order_mw}. Furthermore, for $\mathbf{v}_{i}$,
    \begin{align}
        \hat{\mathbf{v}}_{i} = T_{\omega}(\mathbf{v}_{i}^{(P)}) = \sum_{j=1}^{K_{Q}}\dfrac{\omega_{ij}^{\star}}{p_{i}}\mathbf{v}_{j}^{(Q)}.\label{eq:first_order_smw}
    \end{align}
\end{theorem}

\begin{proof}
    The label distance term in $\mathcal{SMW}$ is independent of $\mathbf{m}_{i}$ and $\mathbf{s}_{i}$, hence the optimality conditions of these variables remain unchanged. Therefore, the first-order optimality condition with respect $\mathbf{v}_{i}$ is,
    \begin{align*}
        \dfrac{\partial \mathcal{SMW}_{2}^{2}}{\partial \mathbf{v}_{i}} = 2\beta\sum_{j=1}^{K_{Q}}\omega_{ij}^{\star}(\mathbf{v}_{i}-\mathbf{v}_{j}^{(Q)}) = 2\beta\biggr{(}p_{i}\mathbf{v}_{i} - \sum_{j=1}^{K_{Q}}\omega_{ij}^{\star}\mathbf{v}_{j}^{(Q)}\biggr{)},
    \end{align*}
    which, for $\beta > 0$, is zero if and only if $\mathbf{v}_{i} = T_{\omega}(\mathbf{v}_{i}^{(P)})$.
\end{proof}

\noindent\emph{Remark.} In equation~\ref{eq:sup_mixture_wasserstein_distance}, we are heuristically adding a label regularization term to the $\mathcal{MW}_{2}$ distance. The actual continuous counterpart (between samples, rather than components) is currently beyond the scope of this paper, but methodologically, this choice remains valid, and is closer to the contributions of~\cite{chen2018optimal}.

\subsection{Mixture Wasserstein Barycenters}\label{sec:mixture_wbary}

In this section, we detail a new algorithm for computing barycenters of \glspl{gmm} under the $\mathcal{MW}_{2}$ and $\mathcal{SMW}_{2}$ metrics. As such, we adapt the definition of~\cite{agueh2011barycenters},
\begin{definition}\label{def:mw2_barycenters}
    Given $C \geq 1$ \glspl{gmm} $\mathcal{P} = \{P_{c}\}_{c=1}^{C}$, $K_{B} \geq 1$, and a vector of barycentric coordinates $\lambda \in \Delta_{C}$, the $\mathcal{SMW}_{2}$ barycenter is given by,
    \begin{align}
        B^{\star} = \mathcal{B}(\lambda,\mathcal{P}) = \argmin{B \in \text{GMM}_{d}(K_{B})}\biggr{\{}\mathcal{L}(B) = \sum_{c=1}^{C}\lambda_{c}\mathcal{SMW}_{2}(B,P_{c})^{2}\biggr{\}}.\label{eq:mw2_barycenter}
    \end{align}
\end{definition}
When the \glspl{gmm} in $\mathcal{P}$ are unlabeled, one may define, by analogy, a barycenter under the $\mathcal{MW}_{2}$. Henceforth we describe an algorithm for labeled \glspl{gmm}, but its extension for unlabeled \glspl{gmm} is straightforward. Inspired by previous results in empirical Wasserstein barycenters~\cite{cuturi2014fast,montesuma2023learning}, we propose a novel strategy for computing $\mathcal{B}(\lambda,\mathcal{P})$. Our method relies on the analysis of $\theta_{B} = \{(\mathbf{m}_{i}^{(B)}, \mathbf{s}_{i}^{(B)},\mathbf{v}_{i}^{(B)})\}_{i=1}^{K_{B}} \mapsto \sum_{c=1}^{C}\lambda_{c}\mathcal{SMW}_{2}(B,P_{c})^{2}$. First, for a fixed $\theta_{B}$, we find $\omega_{1}^{\star},\cdots,\omega_{C}^{\star}$ transport plans. Then, for fixed transport plans, we solve,
\begin{align*}
    \argmin{\theta_{B}}\mathcal{L}(\theta_{B}) &= \sum_{c=1}^{C}\lambda_{c}\sum_{i=1}^{K_{B}}\sum_{j=1}^{K_{P}}\omega_{c,i,j}^{\star}C_{c,i,j},\\
    \text{where }C_{c,i,j} &= \lVert \mathbf{m}_{i}^{(B)} - \mathbf{m}_{j}^{(P_{c})}\rVert_{2}^{2} + \lVert \mathbf{s}_{i}^{(B)} - \mathbf{s}_{j}^{(P_{c})}\rVert_{2}^{2} + \beta\lVert \mathbf{v}_{i}^{(B)} - \mathbf{v}_{j}^{(P_{c})}\rVert_{2}^{2}
\end{align*}
which can be optimized by taking derivatives with respect $\mathbf{m}^{(B)}_{i}$, $\mathbf{s}^{(B)}_{i}$ and $\mathbf{v}^{(B)}_{i}$. For instance, taking the derivative of $\mathcal{L}(\theta_{B})$ with respect $\mathbf{m}_{i}^{(B)}$,
\begin{align*}
    \dfrac{\partial \mathcal{L}}{\partial\mathbf{m}_{i}^{(B)}} = 2\sum_{c=1}^{C}\lambda_{k}\sum_{j=1}^{K_{P}}\omega_{c,i,j}^{\star}(\mathbf{m}_{i}^{(B)} - \mathbf{m}_{j}^{(P_{c})}) = \dfrac{2}{K_{B}}\mathbf{m}_{i}^{(B)} - 2\sum_{c=1}^{C}\lambda_{c}\sum_{j=1}^{K_{P}}\omega_{c,i,j}^{\star}\mathbf{m}_{j}^{(P_{c})}
\end{align*}
setting the derivative to $0$, one has, $\mathbf{m}_{i}^{(B)} = \sum_{c=1}^{C}\lambda_{c}T_{\omega_{c}^{\star}}(\mathbf{m}_{i}^{(B)}).$ Similar results can be acquired for $\mathbf{s}_{i}^{(B)}$ and $\mathbf{v}_{i}^{(B)}$ by taking the appropriate derivatives. Our strategy is shown in Algorithm~~\ref{alg:gmmot_bary}.

\begin{algorithm}[ht]
  \caption{$\mathcal{SMW}_{2}$ Barycenter of GMMs}
  \label{alg:gmmot_bary}
  \Function{smw\_barycenter($\{(\mathbf{M}^{(P_{c})},\mathbf{S}^{(P_{c})},\mathbf{V}^{(P_{c})})\}_{c=1}^{C}$, $\tau$, $N_{it}$)}{
    $\mathbf{m}_{i}^{(B)} \sim \mathcal{N}(\mathbf{0}, \mathbf{I}_{d})$, $\mathbf{s}_{i}^{(B)} = 1$ and $\mathbf{y}_{i}^{(B)} = \nicefrac{\mathds{1}_{n_{c}}}{n_{c}}$\;
    \While{$|L_{it} - L_{it-1}| \geq \tau$ and $it \leq N_{it}$}{
        \Comment{Compute GMM-OT plans}
        \For{$c=1,\cdots,C$}{
            $\omega^{(c,it)} = \text{GMMOT}(B,P_{c})$\;
        }
        \Comment{Note: $\mathcal{W}_{2}(B_{i}, P_{c,j})^{2} = \lVert \mathbf{m}_{i}^{(B)} - \mathbf{m}_{j}^{(P_{c})} \rVert_{2}^{2} + \lVert \mathbf{s}_{i}^{(B)} - \mathbf{s}_{j}^{(P_{c})}\rVert_{2}^{2}$}
        $L_{it} = \sum_{c=1}^{C}\lambda_{c}\sum_{i=1}^{K_{B}}\sum_{j=1}^{K_{P}}\omega^{(c,it)}_{ij}\biggr{(}(\mathcal{W}_{2}(B_{i},P_{c,j})^{2} + \beta\lVert \mathbf{v}^{(B)}_{i} - \mathbf{v}^{(P_{c})}_{j} \rVert_{2}^{2})\biggr{)}$\;
        \Comment{Update barycenter parameters}
        $\mathbf{m}^{(B)}_{i} = \sum_{c=1}^{C}\lambda_{c}T_{\omega^{(c,it)}}(\mathbf{m}^{(B)}_{i})$\;
        $\mathbf{s}^{(B)}_{i} = \sum_{c=1}^{C}\lambda_{c}T_{\omega^{(c,it)}}(\mathbf{s}^{(B)}_{i})$\;
        $\mathbf{v}^{(B)}_{i} = \sum_{c=1}^{C}\lambda_{c}T_{\omega^{(c,it)}}(\mathbf{v}^{(B)}_{i})$\;
    }
    \Return{$\mathbf{M}^{(B)}$, $\mathbf{S}^{(B)}$, $\mathbf{V}^{(B)}$}\;
  }
\end{algorithm}

\subsection{Multi-Source Domain Adaptation through GMM-OT}\label{sec:msda_gmm}

In this section, we detail two contributions for \gls{msda} based on \gls{gmm}-\gls{ot}: \gls{gmm}-\gls{wbt} and \gls{gmm}-\gls{dadil}. In both cases, we suppose access to $N$ labeled source \glspl{gmm} $\mathcal{Q}_{S} = \{Q_{S_{\ell}}\}_{\ell=1}^{N}$ and an unlabeled target \gls{gmm} $Q_{T}$. Contrary to the empirical versions of these algorithms~\cite{montesuma2021icassp,montesuma2021cvpr,montesuma2023learning}, we assume that an axis-aligned \gls{gmm} has been learned for each domain, including the target.

\noindent\textbf{GMM-WBT.} The intuition of this algorithm is transforming the \gls{msda} scenario into a single-source one, by first calculating a Wasserstein barycenter of $B = \mathcal{B}(\mathds{1}_{N} / N; \mathcal{Q}_{S})$. After this step, \gls{wbt} solves a single-source problem between $B$ and $Q_{T}$. When each $Q_{S_{\ell}}$ is a \gls{gmm}, the parameters of $B$ are estimated through algorithm~\ref{alg:gmmot_bary}. Next, one solves for $\omega^{(T)} = \text{GMMOT}(B, Q_{T})$, so that the parameters of $B$ are transported towards $Q_{T}$ using theorems~\ref{thm:first_order_mw2} and~\ref{thm:first_order_smw2},
\begin{align}
    \hat{\mathbf{m}}_{i}^{(Q_{T})} = K_{B}\sum_{j=1}^{K_{T}}\omega_{ij}^{(T)}\mathbf{m}_{j}^{(Q_{T})}\text{, and, }\hat{\mathbf{s}}_{i}^{(Q_{T})} = K_{B}\sum_{j=1}^{K_{T}}\omega_{ij}^{(T)}\mathbf{s}_{j}^{(Q_{T})}.\label{eq:matched_gmm}
\end{align}
With a labeled \gls{gmm}, $\{\hat{\mathbf{m}}_{i}^{(Q_{T})},
\hat{\mathbf{s}}_{i}^{(Q_{T})},\mathbf{v}_{i}^{(B)}\}_{i=1}^{K_{B}}$, on the target domain, we can learn a classifier on the target domain as explained in section~\ref{sec:supervised_gmm}.

\begin{algorithm}[ht]
  \caption{GMM-Dataset Dictionary Learning}
  \label{alg:gmm_dadil}
  \Function{gmm\_dadil($\{(\mathbf{M}^{(Q_{S_{\ell}})},\mathbf{S}^{(Q_{S_{\ell}})},\mathbf{V}^{(Q_{S_{\ell}})})\}_{\ell=1}^{N}$, $\{(\mathbf{M}^{(Q_{T})},\mathbf{S}^{(Q_{T})})\}$, $N_{it}$, $\eta$)}{
    \Comment{Initialization.}
    $\mathbf{m}_{i}^{(P_{k})} \sim \mathcal{N}(\mathbf{0}, \mathbf{I}_{d})$, $\mathbf{s}_{i}^{(P_{k})} := 1$, $\mathbf{u}_{i}^{(P_{k})} := \nicefrac{1}{n_{c}}$, and $\lambda_{\ell} = \nicefrac{1}{K}$\;
    \For{$it=1,\cdots,N_{it}$}{
        $L \leftarrow 0$\;
        \Comment{Change of variables}
        $\mathbf{v}_{i}^{(P_{k})} \leftarrow \text{softmax}(\mathbf{u}_{i}^{(P_{k})})$\;
        \Comment{Evaluate supervised loss on sources}
        \For{$\ell=1,\cdots,N$}{
            $L \leftarrow L + \mathcal{SMW}_{2}(Q_{S_{\ell}},\mathcal{B}(\lambda_{\ell}, \mathcal{P}))^{2}$\;
        }
        \Comment{Evaluate unsupervised loss on targets}
        $L \leftarrow L + \mathcal{MW}_{2}(Q_{T},\mathcal{B}(\lambda_{T}, \mathcal{P}))^{2}$\;
        \Comment{Gradient step}
        $\mathbf{m}_{j}^{(P_{k})} \leftarrow \mathbf{m}_{j}^{(P_{k})} - \eta \nicefrac{\partial L}{\partial \mathbf{m}^{(P_{k})}_{j}}$\;
        $\mathbf{u}_{j}^{(P_{k})} \leftarrow \mathbf{u}_{j}^{(P_{k})} - \eta \nicefrac{\partial L}{\partial \mathbf{u}^{(P_{k})}_{j}}$\;
        \Comment{Note: we project variables $\mathbf{s}$ and $\lambda$.}
        $\mathbf{s}_{j}^{(P_{k})} \leftarrow \text{proj}_{\mathbb{R}^{d}_{+}}(\mathbf{s}_{j}^{(P_{k})} - \eta \nicefrac{\partial L}{\partial \mathbf{s}^{(P_{k})}_{j}})$\;
        $\lambda_{\ell} \leftarrow \text{proj}_{\Delta_{C}}(\lambda_{\ell} - \eta \nicefrac{\partial L}{\partial \lambda_{\ell}})$\;
    }
    \Return{$\Lambda,\mathcal{P}$}\;
  }
\end{algorithm}
\noindent\textbf{GMM-DaDiL.} Our second algorithm consists of a parametric version for the \gls{dadil} algorithm of~\cite{montesuma2023learning}. The idea is to replace the atoms in $\mathcal{P} = \{\hat{P}_{c}\}_{c=1}^{C}$ by \glspl{gmm} parametrized through $\Theta_{P} = \{(\mathbf{M}^{(P_{c})}, \mathbf{S}^{(P_{c})}, \mathbf{V}^{(P_{c})})\}_{c=1}^{C}$. Learning a dictionary is thus equivalent to estimating these parameters, that is,
\begin{align}
    (\Lambda^{\star}, \Theta_{P}^{\star}) &= \argmin{\Lambda,\Theta_{P}} \mathcal{MW}_{2}(Q_{T}, \mathcal{B}(\lambda_{T},\mathcal{P}))^{2} + \sum_{\ell=1}^{N}\mathcal{SMW}_{2}(Q_{\ell}, \mathcal{B}(\lambda_{\ell}; \mathcal{P}))^{2}.\label{eq:gmm_dadil}
\end{align}

While eq.~\ref{eq:gmm_dadil} does not have a closed-form solution, we optimize it through gradient descent. An advantage of the \gls{gmm} modeling is that this optimization problem involves far less variables than \gls{dadil}, hence we do not resort to mini-batches. We detail our strategy in Algorithm~\ref{alg:gmm_dadil}. Note that we need to enforce 3 kinds of constraints: (i) $\mathbf{s}_{i}^{(P_{c})} \in \mathbb{R}^{d}_{+}$, (ii) $\lambda_{\ell} \in \Delta_{C}$ and (iii) $\mathbf{y}_{i}^{(P_{c})} \in \Delta_{n_{cl}}$. For (i) and (ii), we use orthogonal projections into $\mathbb{R}^{d}_{+}$ and $\Delta_{C}$ respectively. We additionally set $\mathbf{s}_{i}^{(P_{c})} \geq s_{min}$ for numerical stability. For (iii), we perform a change of variables $\mathbf{y}_{i}^{(P_{c})} = \text{softmax}(\mathbf{u}_{i}^{(P_{c})})$.

Once the dictionary $(\Lambda, \mathcal{P})$ is learned, we are able to reconstruct the domains in \gls{msda} via the barycenter $\mathcal{B}(\lambda; \mathcal{P})$. We are especially interested in the target reconstruction $\lambda_{T}$, i.e., $\mathcal{B}(\lambda_{T},\mathcal{P})$. This barycenter is a labeled \gls{gmm} (as we show in figure~\ref{fig:ToyExampleDaDiL}b). As a result, we can obtain labeled samples from this \gls{gmm}, then use them to train a classifier that works on the target domain.

The computational complexity of an optimization step of algorithm~\ref{alg:gmm_dadil} corresponds to $\mathcal{O}(N \times N_{it} \times C \times K^{3}\log K)$, i.e., we calculate $N$ barycenters of $C$ atoms. One should compare this complexity with that of \gls{dadil}, i.e., $\mathcal{O}(N \times N_{it} \times M \times C \times n_{b}^{3}\log n_{b})$, where $M = \lceil \nicefrac{n}{n_{b}} \rceil$ is the number of mini-batches sampled at each iteration. In our experiments in section~\ref{sec:experiments}, we show that we achieve state-of-the-art performance with $K$ on the same order of magnitude as $n_{b}$ (e.g., a few hundred Gaussian components). As a result, we achieve a speed-up on the order of $M$ while solving an exact \gls{ot} problem (see Figure~\ref{fig:additionalexperiments} (c) below).
\section{Experiments}\label{sec:experiments}

\subsection{Toy Example}

In this section, we explore \gls{gmm}-\gls{wbt} and \gls{gmm}-\gls{dadil} in the context of a toy example. We generate 4 datasets over $\mathbb{R}^{d}$, by gradually shifting and deforming initial measure through an affine mapping. In figure~\ref{fig:ToyExampleSamples} (a) we show the generated datasets. Starting with \gls{gmm}-\gls{wbt}, figure~\ref{fig:ToyExampleSamples}  (b) shows the learned \glspl{gmm} for each dataset, in which the target \gls{gmm} is not labeled. The barycenter of $\mathcal{Q}_{S} = \{Q_{S_{\ell}}\}_{\ell=1}^{N}$ is shown in figure~\ref{fig:ToyExampleWBT} (bottom-left). This barycenter is labeled. As a result, we may transfer its parameters to the target domain through \gls{gmm}-\gls{ot} (upper part), which leads to a labeled \gls{gmm} in the target domain.

\begin{figure}[ht]
    \centering
    \begin{minipage}{0.55\linewidth}
        \begin{subfigure}{\linewidth}
            \centering
            \includegraphics[width=\linewidth]{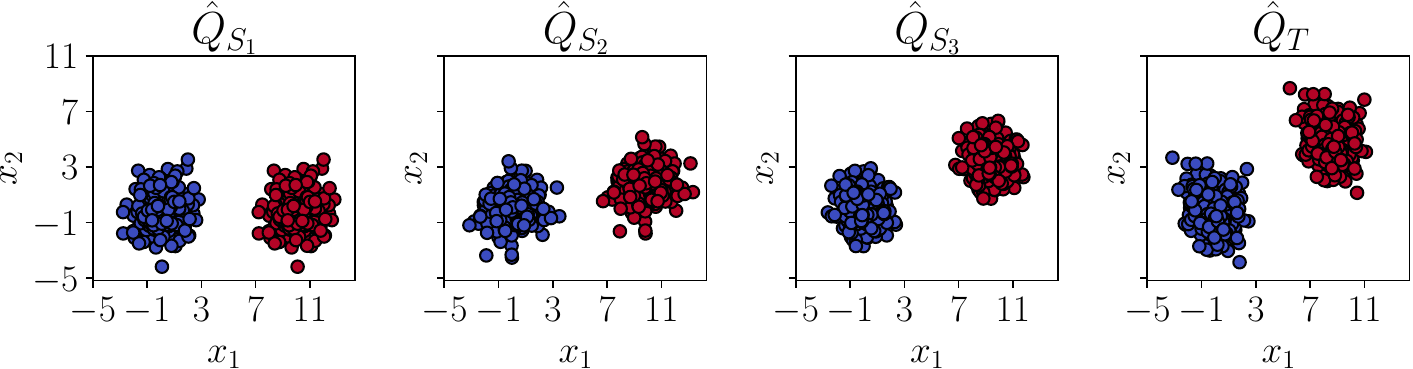}
            \caption{Data.}
        \end{subfigure}
        \begin{subfigure}{\linewidth}
            \centering
            \includegraphics[width=\linewidth]{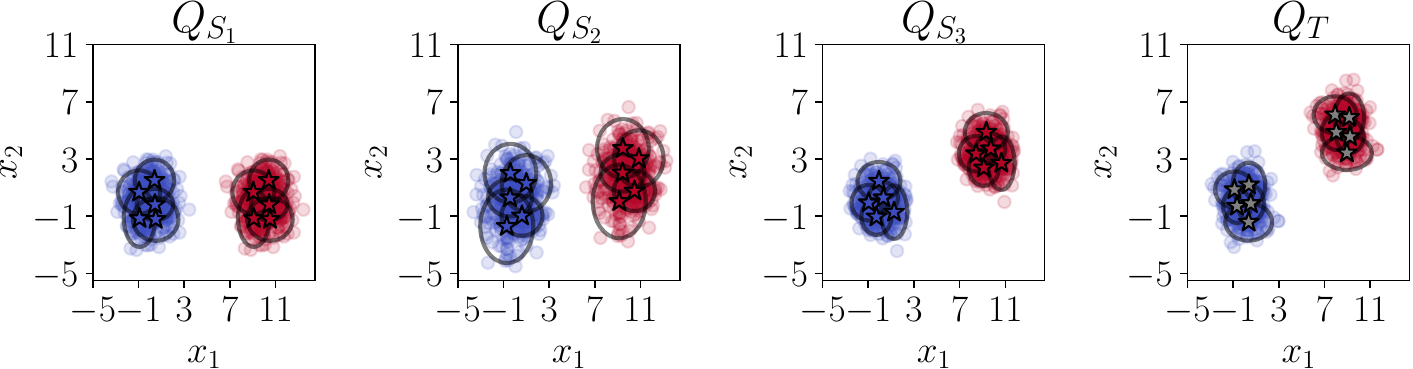}
            \caption{\glspl{gmm}.}
        \end{subfigure}
    \end{minipage}\hfill
    \begin{minipage}{0.4\linewidth}
        \begin{subfigure}{\linewidth}
            \centering
            \includegraphics[width=0.8\linewidth]{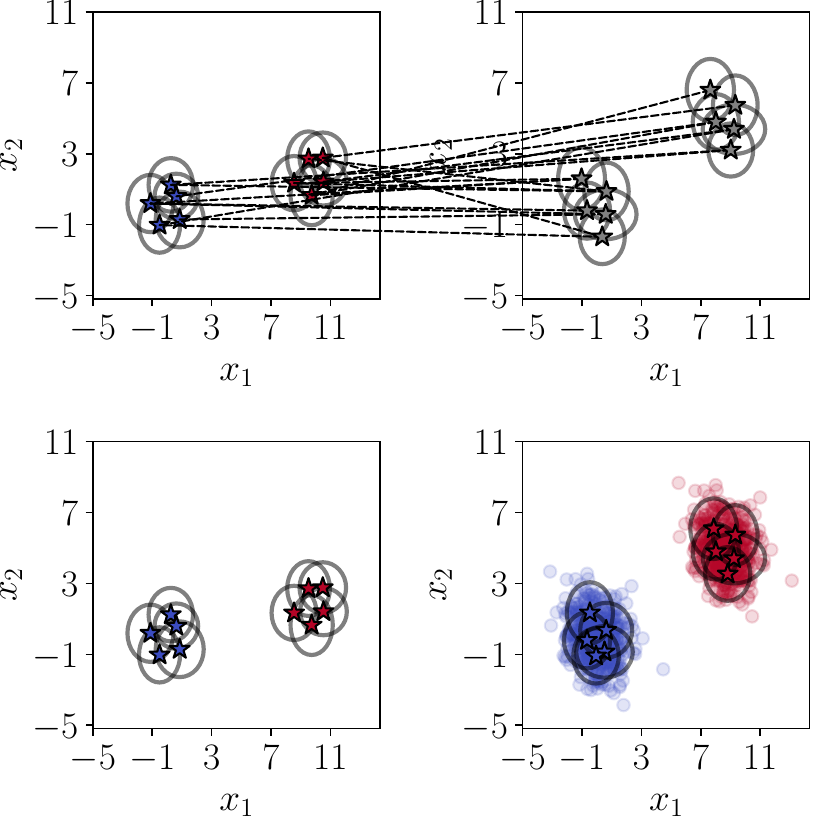}
            \caption{\gls{gmm}-\gls{wbt}}
            \label{fig:ToyExampleWBT}
        \end{subfigure}
    \end{minipage}
    \caption{\textbf{Data and \glspl{gmm} used in the toy experiment.} In (a) Each of these datasets was generated by applying an affine transformation to an initial dataset. In (b), we show an axis-aligned \gls{gmm} fitted to the data via \gls{em}. In (c), we show a summary of \gls{gmm}-\gls{wbt}, where show the \gls{ot} plan between components (upper part) between $B$ (left) and $Q_{T}$ (right). The resulting labeled \gls{gmm} is shown in the lower right part of (c).}
    \label{fig:ToyExampleSamples}
\end{figure}

Next, we show in figure~\ref{fig:ToyExampleDaDiL} a summary for the \gls{gmm}-\gls{dadil} optimization process (figure~\ref{fig:opt_summary}), and the reconstruction of target domain \glspl{gmm} (figure~\ref{fig:reconstructions}). Note that, as the training progresses, the reconstruction error and the negative log-likelihood of the \glspl{gmm} decrease. As a result, \gls{gmm}-\gls{dadil} produces accurate, labeled \glspl{gmm} for each domain. We provide further examples on the \glspl{gmm}-\gls{dadil} optimization in the supplementary materials. Next, we present our results on \gls{msda} benchmarks.

\begin{figure}[ht]
    \centering
    \begin{subfigure}{0.43\linewidth}
        \includegraphics[width=\linewidth]{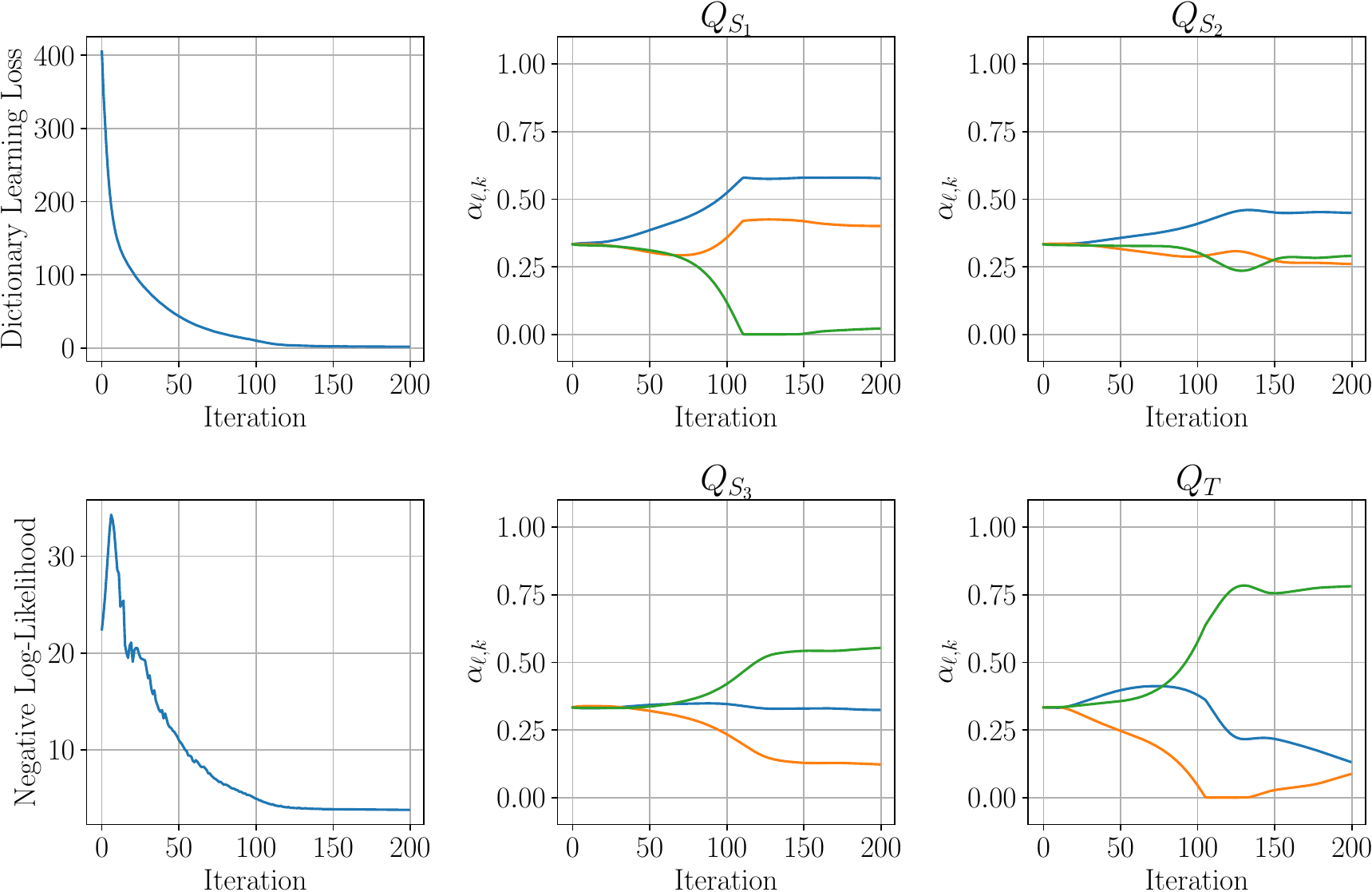}
        \caption{Optimization summary.}
       \label{fig:opt_summary}
   \end{subfigure}\hfill
   \begin{subfigure}{0.54\linewidth}
       \includegraphics[width=\linewidth]{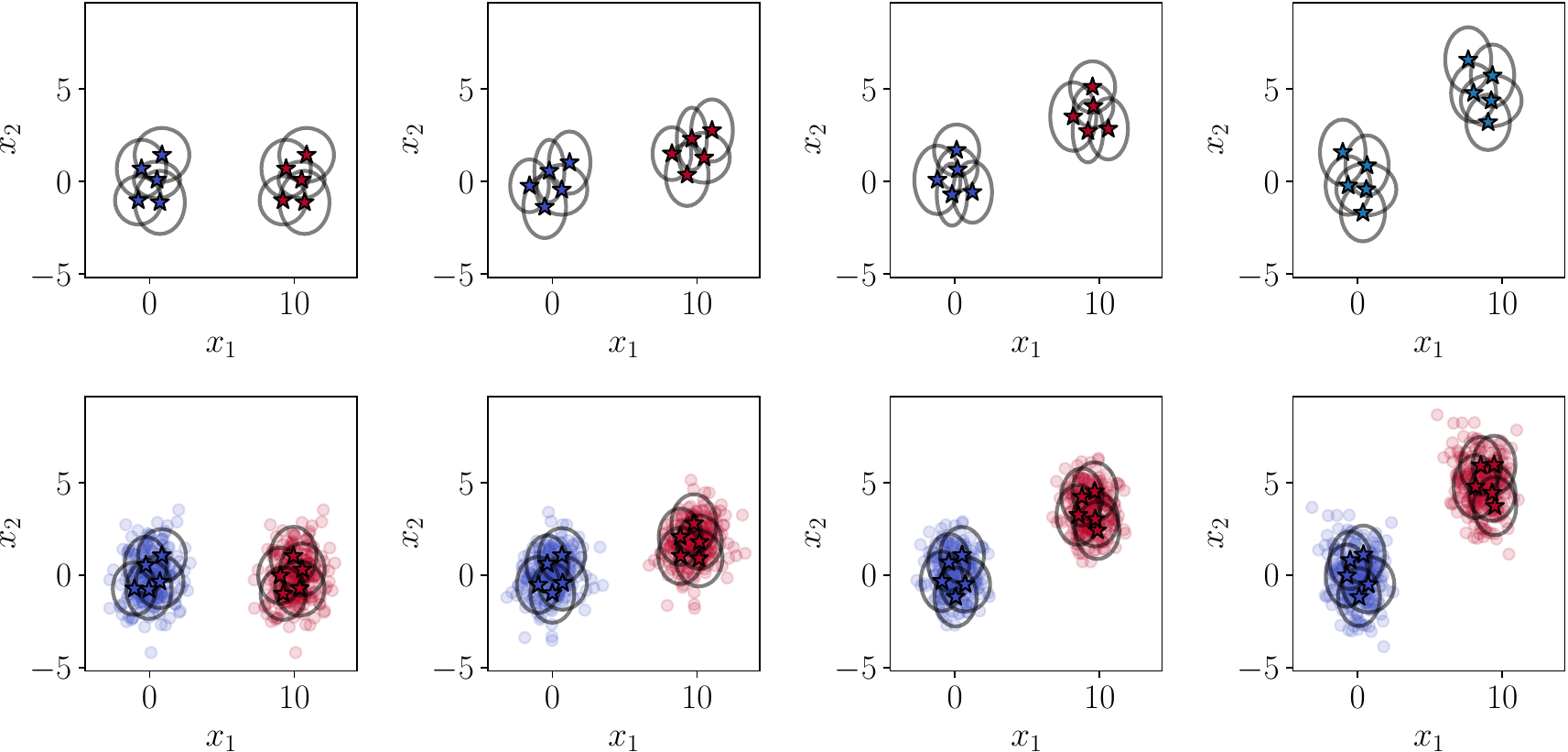}
       \caption{Reconstructions ($it=200$).}
       \label{fig:reconstructions}
   \end{subfigure}
   \caption{\textbf{Optimization and reconstruction summaries using \gls{gmm}-\gls{dadil}.} In (a), we show the evolution of loss, negative log-likelihood and barycentric coordinates (i.e., $\lambda_{\ell}$) over the course of optimization. In (b), we show the reconstructed \glspl{gmm} (i.e., $\mathcal{B}(\lambda_{\ell},\mathcal{P})$) when the algorithm converges.}
   \label{fig:ToyExampleDaDiL}
\end{figure}

\subsection{Multi-Source Domain Adaptation}

We compare our method to prior art. We focus on \gls{ot}-based methods, such as WJDOT~\cite{turrisi2022multi}, WBT~\cite{montesuma2021icassp,montesuma2021cvpr} and DaDiL~\cite{montesuma2023learning}. For completeness, we include recent strategies on deep \gls{msda} that update the encoder network during the adaptation process, rather than using pre-extracted features. These are M$^{3}$SDA~\cite{peng2019moment}, LtC-MSDA~\cite{wang2020learning}, KD3A~\cite{feng2021kd3a} and Co-MDA~\cite{liu2023co}. We establish our comparison on 4 benchmarks, divided between visual domain adaptation (Office31~\cite{saenko2010adapting}, Office-Home~\cite{venkateswara2017deep}) and cross-domain fault diagnosis (TEP~\cite{montesuma2023multi} and CWRU). See table~\ref{tab:datasets} for further details.
\begin{table}[ht]
    \centering
    \caption{Overview of benchmarks used in our experiments.}
    \begin{tabular}{lcccccc}
         \toprule
         Benchmark & Backbone & Problem & \# Samples & \# Domains & \# Classes & \# Features  \\
         \midrule
         Office 31 & ResNet50 & Object Recognition & 3287 & 3 & 31 & 2048\\
         Office-Home & ResNet101 & Object Recognition & 15500 & 4 & 65 & 2048\\
         TEP & CNN &  Fault Diagnosis & 17289 & 6 & 29 & 128\\
         CWRU & MLP & Fault Diagnosis & 24000 & 3 & 10 & 256\\
         \bottomrule
    \end{tabular}
    \label{tab:datasets}
\end{table}

As with previous works on \gls{ot}-based \gls{msda}, we perform domain adaptation on pre-extracted features. As such, we pre-train a neural network (called backbone) on the concatenation of source domain data, then we use it to extract the features from each domain. For visual adaptation tasks, we use ResNets~\cite{he2016deep}, while for fault diagnosis, we use a CNN and a multi-layer perceptron, as in~\cite{montesuma2023learning,montesuma2023multi}. We summarize our results in table~\ref{tab:msda_results}.

\begin{table}[ht]
    \centering
    \caption{Classification accuracy of domain adaptation methods divided by benchmark. $\star$, $\dag$, $\ddag$ and $\S$ denote results from~\cite{montesuma2023learning,montesuma2023multi,liu2023co,castellon2023federated}, respectively.}
    \begin{subtable}{0.305\linewidth}
        \resizebox{\linewidth}{!}{\begin{tabular}{lcccc>{\columncolor[gray]{0.9}}c}
        \toprule
        Algorithm & Ar & Cl & Pr & Rw & \scriptsize{Avg. $\uparrow$} \\
        \midrule
        ResNet101 & 72.90 & 62.20 & 83.70 & 85.00 & 75.95\\
        \midrule
        M$^{3}$SDA & 71.13 & 61.41 & 80.18 & 80.64 & 73.34\\
        LtC-MSDA & 74.52 & 60.56 & 85.52 & 83.63 & 76.05\\
        KD3A & 73.80 & 63.10 & 84.30 & 83.50 & 76.17 \\
        Co-MDA$^{\ddag}$ & 74.40 & 64.00 & 85.30 & 83.90 & 76.90\\
        \midrule
        WJDOT & 74.28 & 63.80 & 83.78 & 84.52 & 76.59 \\
        WBT & 75.72 & 63.80 & 84.23 & 84.63 & 77.09 \\
        DaDiL-E & \textbf{77.16} & 64.95 & 85.47 & 84.97 & \underline{78.14}\\
        DaDiL-R & \underline{75.92} & \underline{64.83} & 85.36 & \textbf{85.32} & 77.86\\
        \midrule
        GMM-WBT & 75.31 & 64.26 & \textbf{86.71} & \underline{85.21} & 77.87\\
        GMM-DaDiL & \textbf{77.16} & \textbf{66.21} & \underline{86.15} & \textbf{85.32} & \textbf{78.81} \\
        \bottomrule
    \end{tabular}}
        \caption{Office-Home.}
    \end{subtable}\hfill
    \begin{subtable}{0.265\linewidth}
        \resizebox{\linewidth}{!}{\begin{tabular}{lccc>{\columncolor[gray]{0.9}}c}
            \toprule
            Algorithm & A & D & W & \scriptsize{Avg. $\uparrow$} \\
            \midrule
            ResNet50 & 67.50 & 95.00 & 96.83 & 86.40\\
            \midrule
            M$^{3}$SDA & 66.75 & 97.00 & 96.83 & 86.86 \\
            LtC-MSDA & 66.82 & 100.00 & 97.12 & 87.98\\
            KD3A & 65.20 & \textbf{100.0} & 98.70 & 87.96 \\
            Co-MDA & 64.80 & \underline{99.83} & 98.70 & 87.83\\
            \midrule
            WJDOT & 67.77 & 97.32 & 95.32 & 86.80 \\
            WBT & 67.94 & 98.21 & 97.66 & 87.93\\
            DaDiL-E & 70.55 & \textbf{100.00} & \underline{98.83} & 89.79\\
            DaDiL-R & \underline{70.90} & \textbf{100.00} & \underline{98.83} & \textbf{89.91}\\
            \midrule
            GMM-WBT & 70.13 & 99.11 & 96.49 & 88.54 \\
            GMM-DaDiL & \textbf{72.47} & \textbf{100.0} & \textbf{99.41} & \textbf{90.63} \\
            \bottomrule
        \end{tabular}}
        \caption{Office 31.}
    \end{subtable}\hfill
    \begin{subtable}{0.4\linewidth}
        \resizebox{\linewidth}{!}{\begin{tabular}{lccc>{\columncolor[gray]{0.9}}c}
            \toprule
            Algorithm & \scriptsize{A} & \scriptsize{B} & \scriptsize{C} & \scriptsize{Avg. $\uparrow$} \\
            \midrule
            MLP$^{\star}$ & 70.90 $\pm$ 0.40 & 79.76 $\pm$ 0.11 & 72.26 $\pm$ 0.23 & 74.31 \\
            \midrule
            M3SDA & 56.86 $\pm$ 7.31 & 69.81 $\pm$ 0.36 & 61.06 $\pm$ 6.35 & 62.57\\
            LTC-MSDA$^{\star}$ & 82.21 $\pm$ 8.03 & 75.33 $\pm$ 5.91 & 81.04 $\pm$ 5.45 & 79.52\\
            KD3A$^{\S}$ & 81.02 $\pm$ 2.92 & 78.04 $\pm$ 4.05 & 74.64 $\pm$ 5.65 & 77.90 \\
            Co-MDA & 62.66 $\pm$ 0.96 & 55.78 $\pm$ 0.85 & 76.35 $\pm$ 0.79 & 64.93\\
            \midrule
            WJDOT & 99.96 $\pm$ 0.02 & 98.86 $\pm$ 0.55 & \textbf{100.0 $\pm$ 0.00 }& 99.60\\
            WBT$^{\star}$ & 99.28 $\pm$ 0.18 & 79.91 $\pm$ 0.04 & 97.71 $\pm$ 0.76 & 92.30\\
            DaDiL-R$^{\star}$ & \underline{99.86 $\pm$ 0.21} & \underline{99.85 $\pm$ 0.08} & \textbf{100.00} $\pm$ \textbf{0.00}  & \underline{99.90}\\
            DaDiL-E$^{\star}$ & 93.71 $\pm$ 6.50 & 83.63 $\pm$ 4.98 & \underline{99.97} $\pm$ \underline{0.05}  & 92.33\\
            \midrule
            GMM-WBT & \textbf{100.00 $\pm$ 0.00} & \textbf{99.95 $\pm$ 0.07} & \textbf{100.00 $\pm$ 0.00} & \textbf{99.98}\\
            GMM-DaDiL & \textbf{100.00 $\pm$ 0.00} & \textbf{99.95 $\pm$ 0.04} & \textbf{100.00 $\pm$ 0.00} & \textbf{99.98}\\
            \bottomrule
        \end{tabular}}
        \caption{CWRU.}
    \end{subtable}\\
    \begin{subtable}{\linewidth}
        \resizebox{\linewidth}{!}{\begin{tabular}{lcccccc>{\columncolor[gray]{0.9}}c}
        \toprule
        Algorithm & \scriptsize{Mode 1} & \scriptsize{Mode 2} & \scriptsize{Mode 3} & \scriptsize{Mode 4} & \scriptsize{Mode 5} & \scriptsize{Mode 6} & \scriptsize{Avg. $\uparrow$} \\
        \midrule
        CNN$^{\dag}$ & 80.82 $\pm$ 0.96 & 63.69 $\pm$ 1.71 & 87.47 $\pm$ 0.99 & 79.96 $\pm$ 1.07 & 74.44 $\pm$ 1.52 & 84.53 $\pm$ 1.12 & 78.48\\
        \midrule
        M$^{3}$SDA$^{\dag}$ & 81.17 $\pm$ 2.00 & 61.61 $\pm$ 2.71 & 79.99 $\pm$ 2.71 & 79.12 $\pm$ 2.41 & 75.16 $\pm$ 3.01 & 78.91 $\pm$ 3.24 & 75.99\\
        KD3A$^{\S}$ & 72.52 $\pm$ 3.04 & 18.96 $\pm$ 4.54 & 81.02 $\pm$ 2.40 & 74.42 $\pm$ 1.60 & 67.18 $\pm$ 2.37 & 78.22 $\pm$ 2.14 & 65.38\\
        Co-MDA & 64.56 $\pm$ 0.62 & 35.99 $\pm$ 1.21 & 79.66 $\pm$ 1.36 & 72.06 $\pm$ 1.66 & 66.33 $\pm$ 0.97 & 78.91 $\pm$ 1.87 & 66.34\\
        \midrule
        WJDOT & 89.06 $\pm$ 1.34 & 75.60 $\pm$ 1.84 & \textbf{89.99 $\pm$ 0.86} & \underline{89.38 $\pm$ 0.77} & 85.32 $\pm$ 1.29 & 87.43 $\pm$ 1.23 & 86.13\\
        WBT$^{\dag}$ & \textbf{92.38 $\pm$ 0.66} & 73.74 $\pm$ 1.07 & 88.89 $\pm$ 0.85 & \underline{89.38 $\pm$ 1.26} & 85.53 $\pm$ 1.35 & 86.60 $\pm$ 1.63 & 86.09\\
        DaDiL-R$^{\ddag}$ & 91.97 $\pm$ 1.22 & \textbf{77.15 $\pm$ 1.32} & 85.41 $\pm$ 1.69 & \textbf{89.39 $\pm$ 1.03} & 84.49 $\pm$ 1.95 & \textbf{88.44 $\pm$ 1.29} & \underline{86.14}\\
        DaDiL-E$^{\ddag}$ & 90.45 $\pm$ 1.02 & \underline{77.08 $\pm$ 1.21} & 86.79 $\pm$ 2.14 & 89.01 $\pm$ 1.35 & 84.04 $\pm$ 3.16 & 87.85 $\pm$ 1.06 &  85.87\\
        \midrule
        GMM-WBT & \underline{92.23 $\pm$ 0.70} & 71.81 $\pm$ 1.78 & 84.72 $\pm$ 1.92 & 89.28 $\pm$ 1.55 & \textbf{87.51 $\pm$ 1.73} & 82.49 $\pm$ 1.81 & 84.67\\
        GMM-DaDiL & 91.72 $\pm$ 1.41 & 76.41 $\pm$ 1.89 & \underline{89.68 $\pm$ 1.49} & 89.18 $\pm$ 1.17 & \underline{86.05 $\pm$ 1.46} & \underline{88.02 $\pm$ 1.12} & \textbf{86.85}\\
        \bottomrule
        \end{tabular}}
        \caption{TEP.}
    \end{subtable}
    \label{tab:msda_results}
\end{table}

First, \gls{ot}-based methods generally outperform other methods in \gls{msda}. Overall, shallow \gls{da} methods solve a simpler task compared to deep \gls{da} methods, as they do not need to update the encoder network during adaptation. Second, the \gls{gmm}-\gls{ot} framework generally improves over using empirical \gls{ot}. For instance, in the CWRU benchmark, \gls{gmm}-\gls{wbt} largely outperforms \gls{wbt}~\cite{montesuma2021icassp,montesuma2021cvpr}. Furthermore, \gls{gmm}-\gls{dadil} outperforms its empirical counterpart on all benchmarks, as well as \gls{gmm}-\gls{wbt}. This point further illustrates the power of dictionary learning in \gls{msda}. Note that, in table~\ref{tab:msda_results} (d), \gls{gmm}-\gls{dadil} manages to have the best average adaptation performance across domains without actually being the best on any single domain. As a consequence, \gls{gmm}-\gls{dadil} enjoys better stability, with respect distribution shift, than previous methods.

\subsection{Lighter, Better, Faster Domain Adaptation}\label{sec:exp_ablations}

\begin{figure}[ht]
    \centering
    \begin{subfigure}{0.3\linewidth}
        \includegraphics[width=\linewidth]{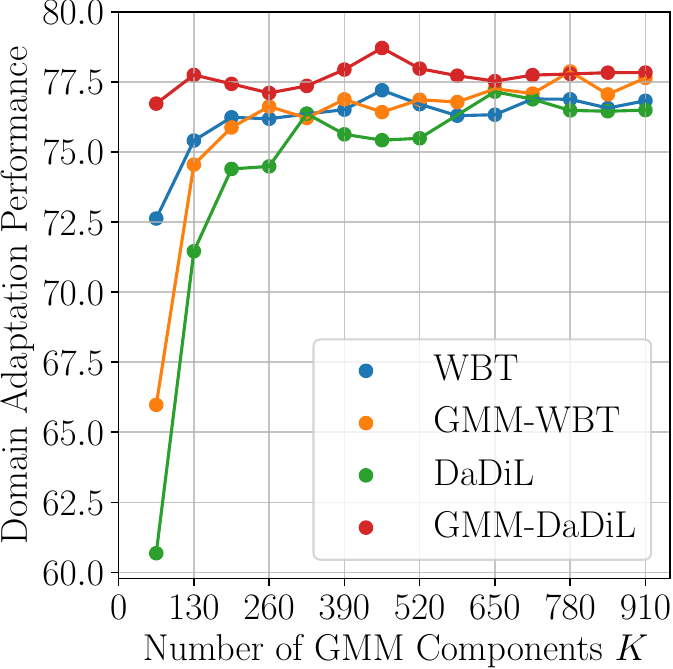}
        \caption{Lighter.}
        \label{fig:perf_analysis}
    \end{subfigure}\hfill
    \begin{subfigure}{0.3\linewidth}
        \includegraphics[width=\linewidth]{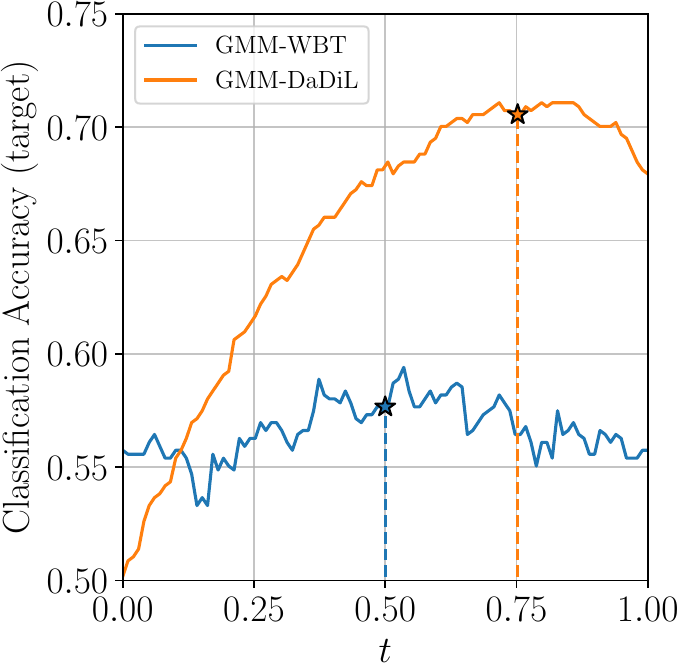}
        \caption{Better.}
        \label{fig:ablation}
    \end{subfigure}\hfill
    \begin{subfigure}{0.3\linewidth}
        \includegraphics[width=\linewidth]{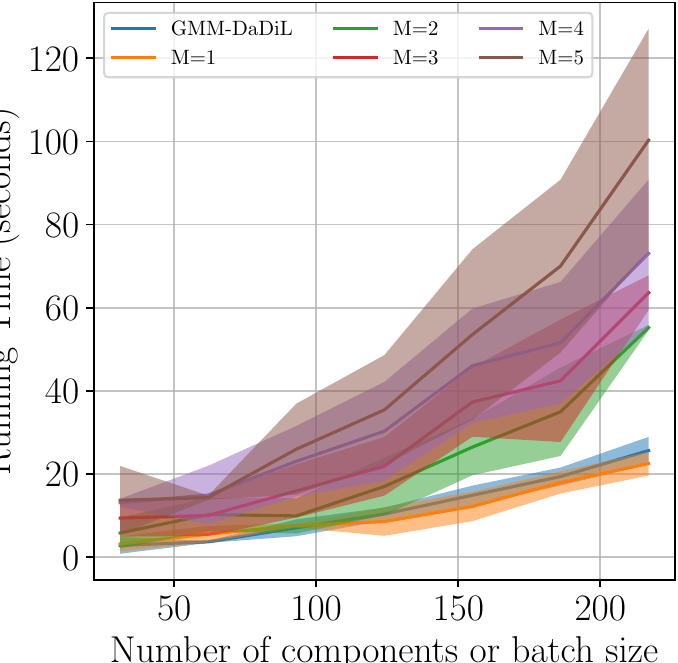}
        \caption{Faster.}
        \label{fig:time_analysis}
    \end{subfigure}
    \caption{\textbf{Lighter, Better, Faster}. In (a), we analyse the performance of interpolations $\mathcal{B}((\lambda_{0},1-\lambda_{0});\mathcal{Q}_{S})$, $\mathcal{Q}_{S} = \{Q_{S_{1}}, Q_{S_{2}}\}$ and $\mathcal{B}((\lambda_{0},1-\lambda_{0});\mathcal{P})$ with learned $\mathcal{P} = \{P_{1}, P_{2}\}$ for \gls{gmm}-\gls{wbt} and \gls{gmm}-\gls{dadil}. In (b), we analyse the efficiency of barycenter-based methods under an increasing number of \gls{gmm} components (number of samples for \gls{dadil} and \gls{wbt}). \gls{gmm}-\gls{dadil} has state-of-the-art performance even for the extreme case where $K = 65$. In (c), we compare the running time of \gls{gmm}-\gls{dadil} with that of \gls{dadil}, as a function of number of components $K$ and batch size $n_{b}$, respectively. This figure illustrates the speedup of \gls{gmm}-\gls{dadil} as the number of samples in \gls{dadil} (and hence, $M = \lceil \nicefrac{n}{n_{b}} \rceil$) increases. Circles represent the average over $5$ independent runs, while the error bars show $2$ times the standard deviation.}
    \label{fig:additionalexperiments}
\end{figure}
Our first experiment illustrates why \gls{gmm}-\gls{dadil} is \textbf{lighter} than previous barycenter-based algorithms, such as \gls{dadil}. In this context, a lighter model needs less parameters to achieve a certain domain adaptation performance. We rank \gls{gmm}-\gls{ot} models by the number of components $K$, and empirical models by the number of samples $n$ in their support. Note that these parameters regulate the complexity of these algorithms. We use the adaptation task $(Cl,Pr$,$Rw)\rightarrow Ar$ from Office-Home for our analysis. We show a comparison in figure~\ref{fig:additionalexperiments} (a). From this figure, we see that \gls{gmm}-\gls{dadil} surpasses all other methods over the entire range $K \in \{65, 130, \cdots, 910\}$. Especially, its empirical counterpart, \gls{dadil}, needs a large number of samples for accurately represent probability measures. Curiously, the performance of \gls{gmm}-\gls{wbt} and \gls{wbt} are quite similar. Indeed, recent studies~\cite{montesuma2023multiicassp} show that Wasserstein barycenters are effective in compressing probability measures with respect the number of their samples. As a result, in this adaptation task, the \gls{gmm} version of \gls{wbt} has similar performance to the empirical version.

Our second experiment illustrates why \gls{gmm}-\gls{otda} provides a \textbf{better} framework for \gls{msda}. We use the adaptation $(D,W)\rightarrow A$ in the Office-31 benchmark as the basis of our experiment. Note that \gls{gmm}-\gls{dadil} reconstructs the target domain via a barycenter $\mathcal{B}(\lambda_{T},\mathcal{P})$, where $\lambda_{T}$ and $\mathcal{P}$ are learned parameters. We thus ablate the learning of $\lambda_{T}$ and $\mathcal{P}$, i.e., we compare it to $\lambda_{T} = (\lambda_{0}, 1 - \lambda_{0})$, $\lambda_{0} \in [0, 1]$ and $\mathcal{Q}_{S} = \{Q_{S_{\ell}}\}_{\ell=1}^{N_{S}}$. This generates a series of measures parametrized by $\lambda_{0}$. To further match $\mathcal{B}(\lambda_{T},\mathcal{Q}_{S})$ with $Q_{T}$, we transport it to $Q_{T}$ through eq.~\ref{eq:matched_gmm}. Note that this corresponds to performing \gls{gmm}-\gls{wbt} with a barycenter calculated with $\lambda_{T}$. The overall experiment is shown in figure~\ref{fig:additionalexperiments} (b). While the performance of \gls{gmm}-\gls{wbt} remains approximately stable, that of \gls{gmm}-\gls{dadil} grows as we move closer to $\lambda_{T}^{\star}$ learned by dictionary learning (blue star). Overall, the interpolation space generated by atoms better captures the distributional shift occurring on the target domain.

Our third experiment shows that \gls{gmm}-\gls{dadil} is \textbf{faster} than \gls{dadil}. We plot the running time of these methods on the Office 31 benchmark, for the $(D,W)\rightarrow A$ adaptation task. The variables that influence the complexity of \gls{gmm} and empirical \gls{dadil} are the number of components $K$ and the batch size $n_{b}$, respectively. For \gls{gmm}-\gls{dadil}, we simply measure its running time for $5$ independent runs of the algorithm (blue curve) for each $K \in \{31, 62, \cdots, 217\}$. For \gls{dadil}, we set $n_{b} \in \{31, 62, \cdots, 217\}$, and set $n = M \times n_{b}$, where $M$ is the number of mini-batches. We measure the performance over $5$ independent runs as well. Other than these parameters, we fix $N_{iter} = 50$ and $C = 3$. As shown in Figure~\ref{fig:additionalexperiments} (c), the running time of \gls{gmm}-\gls{dadil} and \gls{dadil} are essentially equivalent for $M = 1$. For $M > 1$, we have a speedup that is proportional to $M$.

In our fourth experiment, we use the $(B, C) \rightarrow A$ adaptation task of CWRU. We are interested in visualizing the evolution of atoms and reconstructions with respect \gls{dadil} and \gls{gmm}-\gls{dadil} iterations. We visualize this evolution through UMAP~\cite{mcinnes2018umap}, i.e., we concatenate the data from \gls{dadil}'s atoms, i.e., $\mathbf{x}_{i}^{(P_{c,it})}$, so that these are jointly embedded into $\mathbb{R}^{2}$. For \gls{gmm}-\gls{dadil}, we concatenate the mean parameters, i.e., $\mathbf{m}_{i}^{(P_{c,it})}$. We summarize our results in figure~\ref{fig:visualization_dadil}. Overall, as shown in figure~\ref{fig:visualization_dadil} (a -- d), \gls{gmm}-\gls{dadil} optimization is more stable than that of \gls{dadil}, especially since we do not use mini-batches. This remark is also evidenced in the reconstructions in figure~\ref{fig:visualization_dadil} (e -- f). 

\begin{figure}[ht]
    \centering
    \begin{subfigure}{0.2\linewidth}
        \includegraphics[width=\linewidth]{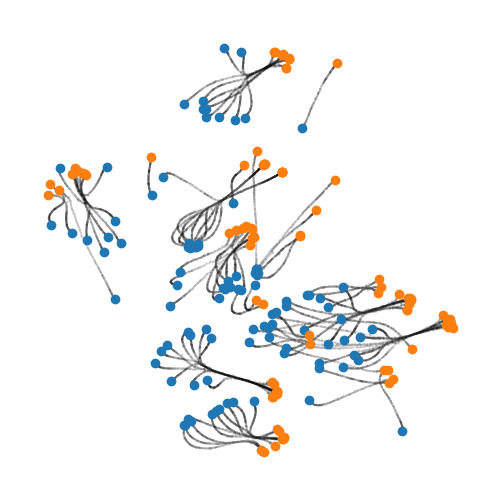}
        \caption{$P_{1,it}$}
    \end{subfigure}\hfill
    \begin{subfigure}{0.2\linewidth}
        \includegraphics[width=\linewidth]{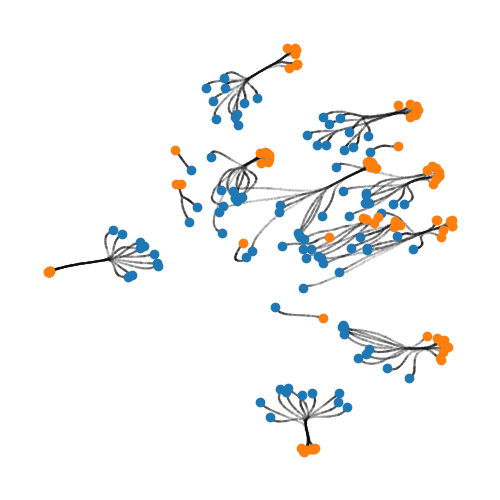}
        \caption{$P_{2,it}$}
    \end{subfigure}\hfill
    \begin{subfigure}{0.2\linewidth}
        \includegraphics[width=\linewidth]{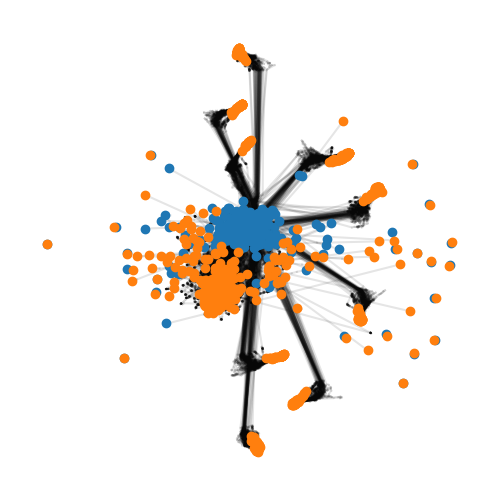}
        \caption{$\hat{P}_{1,it}$}
    \end{subfigure}\hfill
    \begin{subfigure}{0.2\linewidth}
        \includegraphics[width=\linewidth]{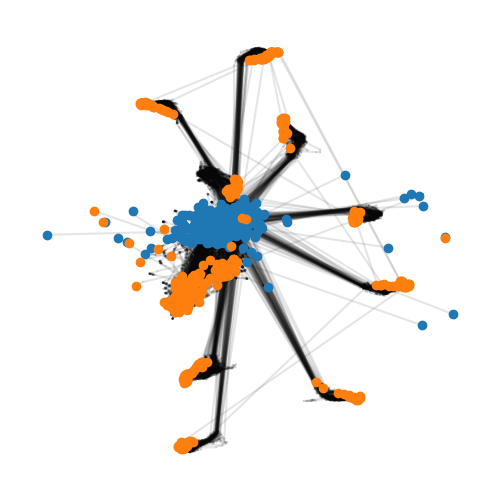}
        \caption{$\hat{P}_{2,it}$}
    \end{subfigure}\\
    \begin{subfigure}{0.2\linewidth}
        \includegraphics[width=\linewidth]{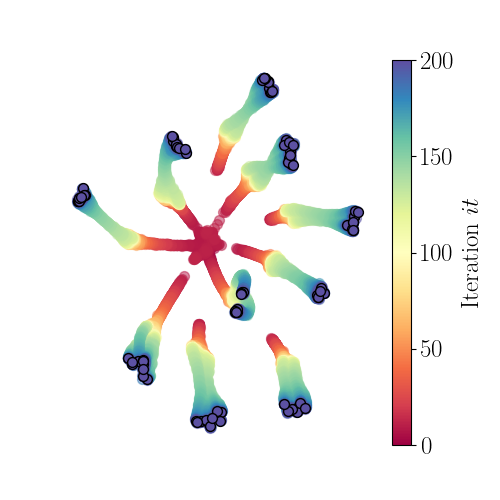}
        \caption{$\mathcal{B}(\lambda_{T};\mathcal{P}_{it})$}
    \end{subfigure}\hfill
    \begin{subfigure}{0.2\linewidth}
        \includegraphics[width=\linewidth]{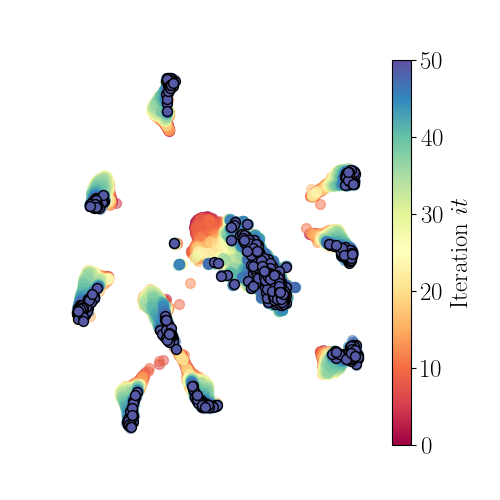}
        \caption{$\mathcal{B}(\lambda_{T};\hat{\mathcal{P}}_{it})$}
    \end{subfigure}\hfill
    \begin{subfigure}{0.2\linewidth}
        \includegraphics[width=\linewidth]{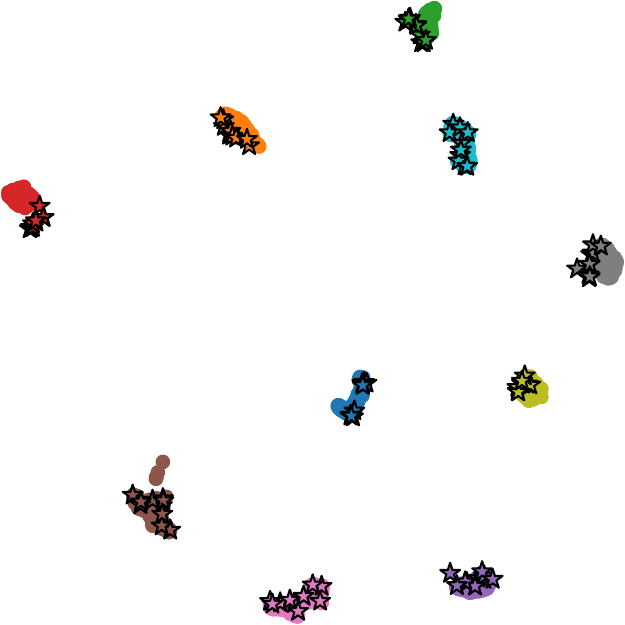}
        \caption{$\mathcal{B}(\lambda_{T};\mathcal{P}^{\star})$}
    \end{subfigure}\hfill
    \begin{subfigure}{0.2\linewidth}
        \includegraphics[width=\linewidth]{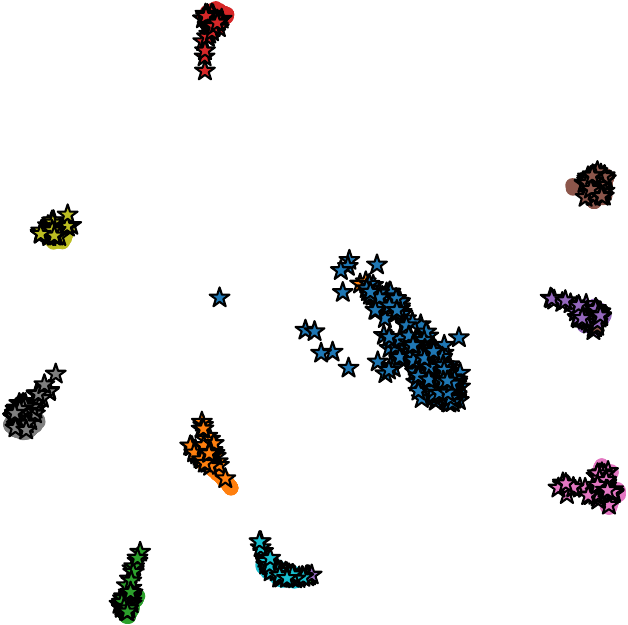}
        \caption{$\mathcal{B}(\lambda_{T};\hat{\mathcal{P}}^{\star})$}
    \end{subfigure}
    \caption{From (a-d), we show the trajectory of atom distributions for \gls{gmm}-\gls{dadil} (a, b) and \gls{dadil} (c, d). Blue and orange points represent the initializations and final vlaues for atoms at convergence. In (e, f), we show the trajectory of barycentric reconstructions for the target domain for these two methods. In (g, h), we show the reconstructions alongside target domain data at convergence.}
    \label{fig:visualization_dadil}
\end{figure}
\section{Conclusion}\label{sec:conclusion}

In this work, we propose a novel framework for MSDA, using GMM-OT~\cite{delon2020wasserstein}. Especially, we propose a novel algorithm for calculating Wasserstein barycenters of \glspl{gmm} (Algorithm~\ref{alg:gmmot_bary}). Based on this algorithm, we propose two new strategies for MSDA: \gls{gmm}-\gls{wbt} and \gls{gmm}-\gls{dadil} (Algorithm~\ref{alg:gmm_dadil}). The first method determines a labeled \gls{gmm} on the target domain by transporting the barycenter of source domain \glspl{gmm} towards the target. The second strategy uses dictionary learning to express each \gls{gmm} in \gls{msda} as the barycenter of learned \glspl{gmm}. Overall, we propose methods that are \textbf{lighter, better, faster} than previous empirical \gls{ot} methods in \gls{msda}.

\bibliographystyle{unsrt}
\bibliography{references}




\end{document}